\documentclass[10pt,journal,cspaper,compsoc]{IEEEtran}

\usepackage{amssymb}
\usepackage{svg}
\usepackage[T1]{fontenc}
\usepackage[noend]{algpseudocode}
\usepackage{algorithmicx}
\usepackage[Algorithm,ruled]{algorithm}
\usepackage{paralist}
\usepackage{amsmath}
\usepackage{amsthm}
\usepackage{enumerate}
\usepackage[color=yellow]{todonotes}
\newtheorem{theorem}{Theorem}[section]
\newtheorem{corollary}{Corollary}[theorem]
\newtheorem{lemma}[theorem]{Lemma}

\newtheorem{definition}{Definition}[section]
\usepackage{mathtools}  
\usepackage{tabulary}
\usepackage{booktabs}
\usepackage{nicefrac}
\usepackage{breqn}
\usepackage[inline]{enumitem}
\usepackage{enumitem}
\setlist{parsep=0pt,listparindent=\parindent}
\usepackage[font=small]{caption}
\makeatletter
\newcommand\footnoteref[1]{\protected@xdef\@thefnmark{\ref{#1}}\@footnotemark}
\makeatother
\usepackage{enumitem}
\allowdisplaybreaks

\ifCLASSOPTIONcompsoc
  \usepackage[nocompress]{cite}
\else
  \usepackage{cite}
\fi

\ifCLASSINFOpdf
\else
\fi

\begin{document}

\title{SafePredict: 
A Meta-Algorithm for Machine Learning That Uses Refusals to Guarantee Correctness }

\author{Mustafa~A.~Kocak,~\IEEEmembership{Student Member,~IEEE,}
        David~{R}amirez,~\IEEEmembership{Member,~IEEE,}
        Elza~Erkip,~\IEEEmembership{Fellow,~IEEE,}
        and~Dennis~E.~Shasha,~\IEEEmembership{Fellow,~ACM}
\IEEEcompsocitemizethanks{\IEEEcompsocthanksitem M. A. Kocak, D. Ramirez and E. Erkip are with the Department
of Electrical and Computer Engineering, NYU Tandon School of Engineering, Brooklyn,
NY, 11201.\protect\\
E-mail: \{kocak, dar550, elza\}@nyu.edu
\IEEEcompsocthanksitem D. E. Shasha is with Courant Institute of Mathematical Sciences
New York University, New York, NY, 10012. \protect \\ E-mail: shasha@courant.nyu.edu}
}

\IEEEtitleabstractindextext{%
\begin{abstract}
\emph{SafePredict} is a novel meta-algorithm that works with any base prediction algorithm for online data to guarantee an arbitrarily chosen correctness rate, $1-\epsilon$, by allowing refusals. Allowing refusals means that the meta-algorithm may refuse to emit a prediction produced by the base algorithm on occasion so that the error rate on non-refused predictions does not exceed $\epsilon$. The SafePredict error bound does not rely on any assumptions on the data distribution or the base predictor. 
When the base predictor happens not to exceed the target error rate $\epsilon$,  SafePredict refuses only a finite number of times. When the error rate of the base predictor changes through time  SafePredict  makes use of a weight-shifting heuristic that adapts to these changes without knowing when the changes occur yet still maintains the correctness guarantee. Empirical results show that (i) SafePredict  compares favorably with state-of-the art confidence based refusal mechanisms which fail to offer robust error guarantees; and (ii) combining SafePredict with such refusal mechanisms can in many cases further reduce the number of refusals. Our software (currently in Python) is included in the supplementary material.
\end{abstract}
}

\maketitle

\IEEEdisplaynontitleabstractindextext
\IEEEpeerreviewmaketitle

\ifCLASSOPTIONcompsoc

\IEEEraisesectionheading{\section{Introduction}\label{sec:introduction}}
\else
\section{Introduction}
\label{sec:introduction}
\fi

Machine learning and statistical inference are the primary building blocks for systems that predict the future from the past. Prediction algorithms have been tailored to fit various applications as varied as analytics \cite{siegel2013predictive}, health care \cite{dua2014machine,han2015building},
and judicial decision making \cite{harbert2013law,mit_tech}. One of the major concerns when utilizing prediction algorithms to automate risk-critical applications is reliability. 

To guarantee an error rate for an overall prediction system, a meta-algorithm should refuse to make a prediction when the meta-algorithm infers that the base prediction algorithm is likely enough to be in error. The implications of refusing to make a prediction may vary according to the application of interest. For example, in a medical diagnosis system, refusing to make a prediction may result in the collection of more information about the patient or a request to a human expert to make a decision based on a more thorough evaluation.

Inspired by the prediction with expert advice framework \cite{cesa2006prediction}, we propose \textit{SafePredict}, an online meta-algorithm that accepts or refuses predictions of a base algorithm depending on the previous performance of the base algorithm. SafePredict asymptotically bounds the error to the desired level without any assumption on the data or the base predictor. When the error rate of the base predictor varies over time, SafePredict will adapt to those changes while preserving the error guarantee. 

\subsection{Prior Work}

The idea of allowing meta-algorithms to refuse to make predictions to reduce the error rate was first introduced by Chow \cite{chow1970optimum}. Chow mainly focused on the classification problem, where data points have an object-label pair independently sampled from a fixed and known probability distribution. Chow showed that the optimum error-refuse trade-off is achieved by a classifier that refuses to predict when the posterior probability of the estimated label given the object is less than a certain threshold. 

The major challenge in applying Chow's work is the need to know the data distribution, which is rarely available. The mainstream assumption is that the available data points are independently sampled from a fixed but unknown distribution. Given these data points, the most common approach is to estimate the underlying distribution and to use it in the  Chow framework. Estimating probability distributions  in high dimensional and complex datasets can be harder even than classification itself, see e.g. \cite{friedman1997bias}. An alternative approach uses confidence-based refusals. One can train a classification algorithm with an arbitrary confidence score (e.g. distance to the decision boundary) and refuse to make a prediction for points with low confidence scores. Examples of this line of work are \cite{hellman1970nearest,landgrebe2006interaction, de2000reject, li2006confidence}. 

Although many practical applications of the refusal framework exist in the literature, e.g. \cite{scheirer2014probability, golfarelli1997error,fumera2003classification,campi2010classification}, theoretical analyses of the suggested methods are relatively rare. {Some} notable exceptions are found in Wegkamp et al. \cite{bartlett2008classification,herbei2006classification,yuan2010classification}, El-Yaniv and Wiener \cite{el2010foundations,wiener2011agnostic} and Cortes et al. \cite{cortes2016learning} which approach the problem from a statistical learning theory perspective and suggest minimizing a linear combination of error and refuse probabilities. Alternatively, reliable agnostic learning, proposed by Kalai et al. \cite{kalai2012reliable}, posits an error threshold and searches for the least refusing predictor from a family of predictors that bounds the error to the error threshold, assuming such a predictor exists.

The related work discussed so far assumes a batch setup, i.e. the algorithm is trained on a fixed set of data points which are independently and identically distributed (i.i.d.). For a more comprehensive literature review of refusal algorithms in the batch setup, please see \cite{herbei2006classification,wiener2013theoretical,cortes2016learning,zhang2017reject} and the references therein. 

In contrast to the batch setup, a meta-algorithm in the online learning framework observes the true outcome after a prediction and then modifies its future behavior. 
In the conformal prediction framework of Vovk et al. \cite{vovk2005algorithmic}, a base algorithm generates confidence scores for each data point. Then the conformal predictor decides to predict or refuse based on these scores. A bound on the error probability and the independence between the error events are guaranteed under the assumptions that the data points are chosen from an exchangeable (essentially i.i.d.) distribution and the base predictor is invariant to the order of the observed data points. For recent developments regarding the classification with refusal problem in the conformal prediction framework we refer the reader to \cite{denis2015consistency,kocak2016conjugate ,lei2014classification} and Chapter 3 of \cite{vovk2005algorithmic}.
Unfortunately, in an online setting the probability of making errors on consecutive predictions may be correlated or the data sequence may have a non-stationary or even an adversarial distribution. Thus, any guarantees that might hold in the independent and identically distributed setting do not directly carry over.

Another approach in the online setting starts with the KWIK (knows what it knows) framework \cite{li2008knows} which removes all assumptions about the data distributions. Instead, KWIK assumes the existence of a perfect predictor (i.e. predictor is always correct) among a set of predictors and aims to find this perfect predictor with a minimum number of refusals. Sayedi et al. \cite{sayedi2010trading} extend this framework by allowing a fixed error budget $k$ and characterizing the minimum number of refusals by keeping the number of errors below $k$. Finally, Zhang et al. \cite{zhang2016extended} relax the perfect predictor assumption to an $l$-bias assumption (i.e. predictor makes at most $l$ errors) and allow the algorithm to refuse in order to achieve an optimal error-refuse trade-off. 

\subsection{Contributions}

This paper makes the following main contributions:
\begin{compactenum}
\item The SafePredict online meta-algorithm can work with any base prediction algorithm to provide an asymptotic error guarantee on the non-refused predictions, without making any assumption about the data or
the base algorithm.
\item All the variants of SafePredict meta-algorithm refuse at most a finite number of times when the base predictor's error rate is below the target error rate $\epsilon$.
\item Adaptive SafePredict meta-algorithm uses weight-shifting and other conventional adaptive procedures to track the error rate of the base predictor in changing environments, thus reducing the number of refusals while preserving the error guarantee.
\item Experiments show that the above theoretical guarantees are achieved in practice and translate to better error performance than other refusal algorithms. The experiments also show that combining SafePredict with previous meta-algorithms can lead to yet fewer refusals in many cases. 
\end{compactenum}

The rest of the paper is organized as follows. Section $2$ presents the problem  and provides a brief introduction to the exponentially weighted average forecasting (EWAF)  \cite{littlestone1989weighted,vavock1990aggregating} expert advice framework.  Section $3$ introduces SafePredict by recasting EWAF as a randomized refusal meta-algorithm and proves its theoretical properties. Section $4$ presents Adaptive SafePredict, a weight-shifting heuristic, to track changes in the error rate of the base algorithm and therefore reduce the number of refusals. Section $5$ presents experiments on real and synthetic data. Section $6$ concludes our work.

\subsection{Notation}

A summary of the notation introduced throughout the paper is given in Table 1. For each quantity, we provide the notation, a brief description and a reference to the section it is  first introduced. 

\begin{table}[!htbp]\caption{Summary of the Notation \\ 
}
\centering 
\begin{tabular}{r p{5cm} c }
\toprule
Not. & Description & Def. in Sec. \\ 
\toprule
$\alpha_t, \beta_t$  & Adaptivity parameters, $\alpha_t \leq w_{P,t+1} \leq \beta_t$. & 4 \\  
$\epsilon$ &   Target error rate. & 2 \\  
$\eta$ &  Learning rate, $\eta > 0$. & 2.1\\ 
$\rho_T$ &   Efficiency of the meta algorithm, $\nicefrac{T^*}{T}$. & 2 \\
$D$ &   Dummy predictor, always refuses, $\hat{y}_{D,t} = \varnothing$ and $l_{D,t} = \epsilon$ for all $t$. & 3\\  
$l_t$ &   Expected loss at time $t$. & 2.1\\    
$L_T$ &  Cumulative expected loss, $\sum_{t=1}^T l_t$. & 2.1 \\ 
$l_{P,t}$ &   Loss of $P$ at time $t$, assume $l_{P,t} \in [0,1]$. & 2\\  
$L_{P,T,t_0}$ &  Partial cumulative loss of $P$ from $t_0$ to $T$, $\sum_{t=t_0+1}^T l_{P,t}$. The third index drops if $t_0=0$&  2 \\     
$L_{P,T}^*$ &   Expected cumulative loss for the meta-algorithm $\sum_{t=1}^Tw_{P,t}l_{P,t}$.  & 2 \\ 
$P$ &  Base predictor. & 2 \\   
$P_i$ &  $i^{th}$ expert in the ensemble, has weight $w_{P_i,t}$ and loss $l_{P_i,t}$.& 2.1 \\ 
$T$ &  Time horizon. & 2\\ 
$T^*$ & Expected value of the number of (non-refused) predictions, $\sum_{t=1}^Tw_{P,t}$.  & 2 \\  
$V^*$ & Variance of the number of (non-refused) predictions, $\sum_{t=1}^Tw_{P,t}(1-w_{P,t})$.  & 2 \\  
$w_{D,t}$ & Weight of the dummy, also the probability of refusal, $1-w_{P,t}$. & 3\\  
$w_{P,t}$ & Probability of making a prediction at time $t$. &  2 \\ 
$\hat{y}_{P,t}$  & Prediction of $P$ at time $t$. & 2 \\  
$\hat{y}_t$  & Filtered prediction, $\hat{y}_t = \hat{y}_{P,t}$ (predict) or  $\hat{y}_t = \varnothing$ (refuse) .& 2 \\  
\vspace{0pt} \\
\bottomrule
\end{tabular}
\label{tab:TableOfNotationForMyResearch}
\end{table}

\section{Problem Setup and Background}

This section introduces the mathematical formulation of online prediction problem with a refusing meta-algorithm and the prediction with expert advice framework.

\subsection{Problem Formulation}

We assume access to a base predictor that produces a label prediction on an observed object. We denote the base predictor as $P$ and a sequence of (object, label) pairs by $(x_1,y_1), (x_2,y_2),\ldots, (x_T,y_T)$ where $T$ is an arbitrary horizon. At each time $t \in \{1,\ldots,T\}$ the base predictor does the following
\begin{compactenum} 
\item Observes the object $x_t$.
\item Predicts the corresponding label $\hat{y}_{P,t}$.
\item Observes the true label $y_t$, and suffers the loss $l_{P,t}$.
\end{compactenum}
In our formulation, we stay agnostic to the data sequence $(x_1,y_1), (x_2,y_2),\ldots, (x_T,y_T)$ and inner workings of the base predictor $P$. We  assume \emph{only} that the loss values are scaled to the unit interval, i.e. $0 \leq l_{P,t} \leq 1~ \forall~t$. 

For example, in a traditional classification task, the labels $y_t$ are  drawn from a finite set and 0-1 loss is used for the loss function, i.e. {$l_{P,t} = 0$ if $\hat{y}_{P,t} = y_t$ and $l_{P,t} = 1$ if $\hat{y}_{P,t} \neq y_t$}. 

Once the base predictor $P$ is chosen, our goal is to design a meta-algorithm $M$ which decides to either follow the prediction made by $P$ \emph{or} refuse to make a prediction for each data point. 
We characterize this meta-algorithm by the following:
\begin{compactitem}
\item \textbf{Parameter}: Target error/loss rate, $\epsilon \in  (0,1)$, is the average loss over time we can tolerate.
\item \textbf{Input}: In full generality, the input of $M$ at time $t$ consists of $x_i, \hat{y}_{P,i}~\forall~i \in \{1,...,t\},$ and $y_j, l_{P,j}~\forall~j \in \{1,...,t-1\}$. Note these are all the observed quantities \emph{before} the label $y_t$ is revealed.
\item \textbf{Output}: A randomized decision to predict (or refuse) at time $t$. We represent the output of $M$ as a probability value $w_{P,t} \in [0,1]$ that is used to decide the final prediction $\hat{y}_t$ as follows:
\[\hat{y}_{t} =   \left\{
\begin{array}{ll}
      \hat{y}_{P,t} & ~~~\textrm{with prob. }w_{P,t} \\
      \varnothing & ~~~\textrm{with prob. }1-w_{P,t} \\
\end{array} 
\right. ,\]
where $\varnothing$ denotes a refusal.
\end{compactitem}
A pictorial description of this general framework is presented in Figure \ref{fig_1}.

We note that $M$ makes a randomized decision at each time point and therefore the number of (non-refused) predictions is a random variable. We compute the expected value $T^*$ of this quantity as \[T^* = \sum_{t=1}^T w_{P,t}.\]
Because we ascribe no loss from refusing to predict, therefore we define the expected cumulative loss of $M$ as
\[L_{P,T}^* = \sum_{t=1}^T w_{P,t}l_{P,t}.\]
Finally, we define the error rate for this randomized meta algorithm by normalizing the cumulative expected loss via the expected number of non-refused predictions, i.e. $  {L_{P,T}^*}/{T^*}$.

Our top priority is to guarantee that the error rate of non-refused predictions made by the meta-algorithm does not exceed the target error rate $\epsilon$ as the number of predictions increases. Following nomenclature introduced by Tukey \cite{tukey1986sunset}, see also \cite{vovk2005algorithmic},	 we call this property  the \textit{validity} of the algorithm. Our goal is to satisfy validity without making any assumptions on the data. An asymptotic definition for the validity is given below.

\begin{definition}
A meta-algorithm $M$ with a given target error $\epsilon$ is called \emph{valid} if 
\[\limsup_{T^*\rightarrow \infty}\frac{L_{P,T}^* }{T^*}  \leq \epsilon . \]
\end{definition}

Next, among valid algorithms we interpret the \textit{efficiency} of an algorithm as the fraction of the predicted data points and define it as follows.

\begin{definition}
\textit{Efficiency} of a meta-algorithm $M$ is denoted by $\rho_T = T^*/T$ and $M$ is called \emph{efficient} if 
\[\lim\inf_{T\rightarrow \infty} \rho_T = 1. \]
\end{definition}

Though the discussion until now has defined validity and efficiency only asymptotically,  we will derive explicit bounds on the excess error rate, i.e. $L_{P,T}^*/ T^* - \epsilon$. However, any non-trivial formal statement about the efficiency of $M$ has to depend on the performance of the base predictor $P$. To keep this dependence minimal, we analyze the efficiency asymptotically for theoretical purposes for all predictors, though experiments show that the efficiency is high for reasonably good predictors. 

\begin{figure}[!htbp] 
\centering
\includegraphics[width=0.85\linewidth]{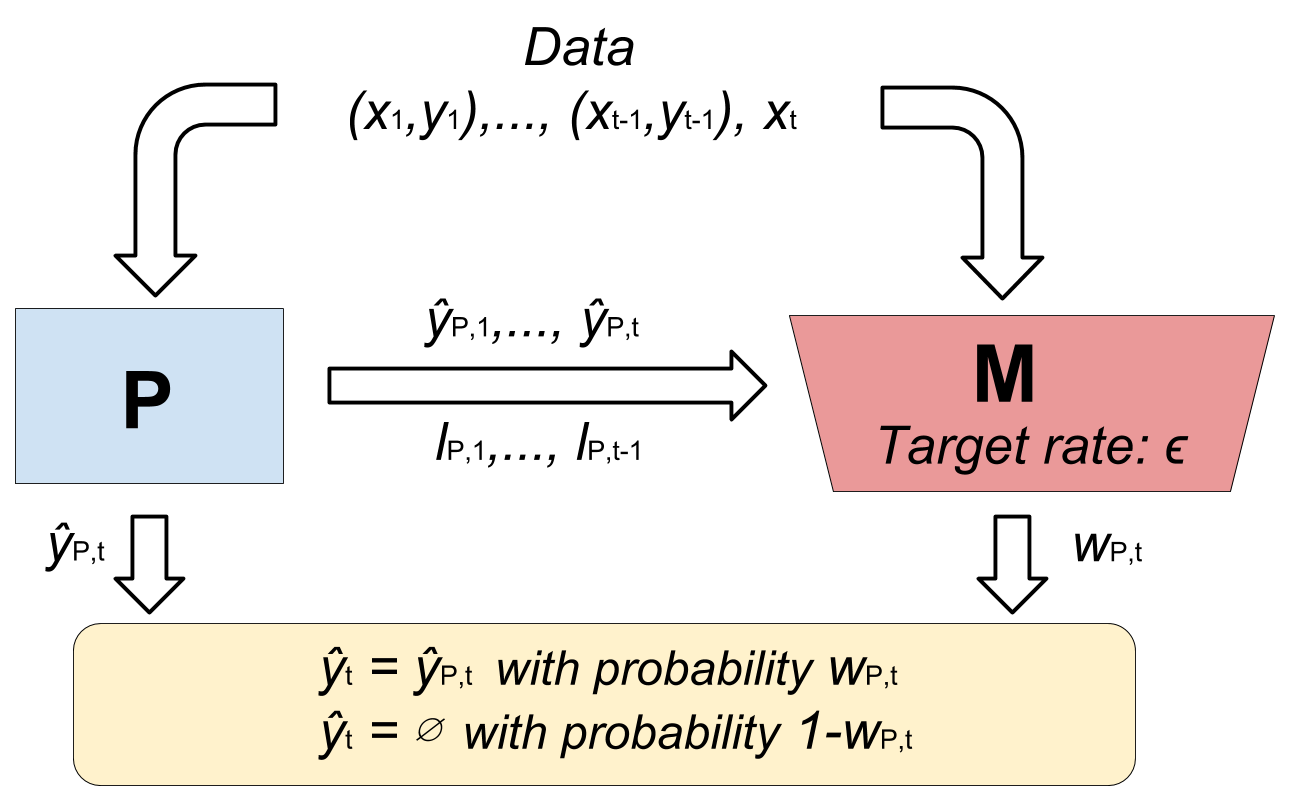}
\caption{\textit{The meta-algorithm, represented by $M$, makes a prediction equivalent to the recommendation of the base  predictor $P$ or refuses to do so for data point $t$ while guaranteeing a target rate $\epsilon$. 
}}
\label{fig_1}
\end{figure}

For the sake of succinctness in the sequel, we introduce the following notation. First, we compute the variance $V^*$ of the number of non-refused predictions with respect to the randomness of the meta-algorithm (i.e. $w_{P,t}$) as
\[V^* = \sum_{t=1}^T w_{P,t}\left(1-w_{P,t}\right).\] 
Next, we denote the cumulative loss for any predictor $P$ (typically for the base predictor) as \[L_{P,T} = \sum_{t=1}^Tl_{P,t}.\] We also introduce a third sub-index $t_0$ for any cumulative quantity to represent the corresponding sum is starting from $t_0+1$, e.g. for the cumulative loss of $P$,  \[L_{P,T,t_0} = \sum_{t=t_0+1}^T l_{P,t}.\]

\subsection{Expert Advice and Exponentially Weighted Average Forecasting (EWAF)}

The prediction with expert advice framework pertains to an online prediction scenario in which one has access to a fixed set of experts (predictors). The goal is to combine these experts' predictions such that the difference between the cumulative losses of the combined predictor and the best expert in the set is minimized. This difference is called \textit{regret}, and the algorithms that {guarantee a regret that grows sub-linearly with the number of predictions are defined as} \textit{Hannan consistent}. The appeal of the expert advice algorithms is the fact that the consistency can be guaranteed even in an adversarial setup. For a comprehensive treatment of the expert advice algorithms, we refer the reader to \cite{cesa2006prediction}.

Exponentially weighted average forecasting (EWAF)  was first introduced by Littlestone and Warmuth \cite{littlestone1989weighted} and by Vovk \cite{vavock1990aggregating} for the expert advice prediction framework. Our meta-algorithm builds on EWAF to provide means to meet any target error rate. Among various expert advice algorithms, we focus on EWAF due to the simplicity of its analysis and strong theoretical guarantees. 

Consider $N$ predictors (experts) $P_1,P_2,\ldots, P_N$ each making a prediction $\hat{y}_{P_1,t}, \ldots, \hat{y}_{P_N,t}$ and each suffering a loss $l_{P_1,t}, \ldots, l_{P_N,t}$ at time $t$. At each time $t$, EWAF outputs one of these predictions as the combined prediction. To do so, EWAF starts with an initial probability distribution over the experts $\left(w_{P_1,1},\ldots, w_{P_N,1}\right)$ and a learning rate $\eta>0$. At each time point $t$, this probability distribution is used to choose an expert and consequently, use the prediction of that expert. The probability distribution is updated after the true label and corresponding losses are revealed. By denoting the probability that the predictor $P_i$ is chosen at time $t$ with $w_{P_i,t}$, the prediction $\hat{y}_t$ is obtained by 
\[\hat{y}_{t} = \hat{y}_{P_i,t}  ~ \textrm{with prob. }w_{P_i,t}, \qquad i=1,\ldots,N.\]
After the true label $y_t$ is revealed, we update these probabilities by multiplying them by an exponential factor that is scaled with the corresponding loss value  and the learning rate, i.e.
\begin{eqnarray} w_{P_i,t+1} & = & \frac{w_{P_i,t} e^{-\eta l_{P_i,t}}}{\sum_{j=1}^N w_{P_j,t} e^{-\eta l_{P_j,t}}}. \label{ewaf_update} \end{eqnarray}
The pseudo-code of the EWAF algorithm is given in Algorithm 1. 
\begin{algorithm}[!htb]
\caption{ \textbf{Exp. Weighted Avg. Forecasting (EWAF)}\\
\textit{Initial weights}: $\left(w_{P_1,1},w_{P_2,1},\ldots, w_{P_N,1}\right)$ \\
\textit{Learning rate}: $\eta$}
\begin{algorithmic}[1]
\For {each $t=1,2,\ldots$}
\State Follow expert $P_i$ with probability $w_{P_i,t}$.
\State Update the weights :
\[ w_{P_i,t+1} = \frac{w_{P_i,t} e^{-\eta l_{P_i,t}}}{\sum_{j=1}^N w_{P_j,t} e^{-\eta l_{P_j,t}}} \]
\EndFor
\end{algorithmic}
\end{algorithm}
Given that EWAF randomly selects a predictor according to a known distribution, we can compute the expected loss for the combined prediction at time $t$ as
\[l_t = \sum_{i=1}^N w_{P_i,t}l_{P_i,t}\]
We denote the expected cumulative loss of the EWAF as \[L_T = \sum_{t=1}^Tl_t\] and the cumulative loss of expert $i, ~\forall~i$ as \[L_{P_i,T} = \sum_{t=1}^T l_{P_i,t}.\] 

The following well-known theorem  establishes an upper bound to the expected loss for EWAF, implying that the regret of EWAF is bounded by $O(\sqrt{T \log N})$.

\begin{theorem} (Theorem 2.2 \cite{cesa2006prediction})
For any learning rate $\eta > 0$, and the initial weights $\left(w_{P_1,1},\ldots, w_{P_N,1}\right)$, we get the following bound for the expected cumulative loss of the EWAF
\begin{eqnarray}
L_T & \leq & L_{P_i,T} - \frac{\log w_{P_i,1}}{\eta} + \frac{\eta T}{8} \label{eq:b1}
\end{eqnarray}
for all $i= 1,\ldots, N$.

If we choose $w_{P_i,1} = 1/N$ for all $i$ and minimize the right hand side (RHS) of the bound above with respective to $\eta$, we get,
\begin{eqnarray}
 L_T & \leq & L_{P_i,T} + \sqrt{\frac{T\log N}{2}}
\label{eq:bound}
\end{eqnarray}
which is achieved for  $\eta = \sqrt{8\log(N) /T}$.
\end{theorem}
The bound given in eq. (\ref{eq:bound}) is optimal in the sense that there exists a matching lower bound in the worst case (see Chapter 3.7 \cite{cesa2006prediction}). Since the bound on the total sum of the expected loss holds for any predictor, it follows that the bound holds for the predictor with the least cumulative loss. Therefore, the bound in eq. (\ref{eq:bound}) limits how far EWAF is from optimality in terms of the sum of the expected loss.

Applying the bound in eq. (\ref{eq:bound}) is insufficient to establish the validity of our algorithm, since eq. (\ref{eq:bound}) bounds regret only with $O\left(\sqrt{T}\right)$ while validity requires a sub-linear bound with $T^*$. In Section 3.1 we present a refined analysis with an appropriate learning rate to establish validity.

\section{Reliable Prediction with Refusals}

We now introduce our meta-algorithm, \textit{SafePredict}, which uses EWAF to convert any given base predictor $P$ into a refusing one with a guaranteed error bound.
Furthermore, we provide a theoretical analysis to establish the asymptotic validity and efficiency of SafePredict. Then, we employ the well-known doubling trick (see for instance Exercise 2.8 \cite{cesa2006prediction} or Theorem 7.7 \cite{mohri2012foundations}) to choose the learning rate and extend our validity analysis accordingly.

In order to achieve validity, we introduce a trivial predictor which can meet the target error rate by refusing to predict all the time. We refer to this particular predictor as the ``dummy predictor'' and denote it by $D$, i.e. $\hat{y}_{D,t}=\varnothing~\forall~t$. We also assume that the predictor $D$ suffers a constant loss $\epsilon$ for each time $t$, i.e. $l_{D,t}=\epsilon$ and $L_{D,t} = \epsilon t~\forall~t$. 

SafePredict is obtained by employing an instance of the EWAF algorithm that runs on the ensemble $\{D,P\}$, to decide either to refuse or predict. Therefore at each step, SafePredict follows the base predictor with probability $w_{P,t}$ and computes the prediction probability for the next round after observing the corresponding loss of $P$ as
\begin{eqnarray}
w_{P,t+1} & = & \frac{w_{P,t}e^{-\eta l_{P,t}}}{w_{P,t}e^{-\eta l_{P,t}} + w_{D,t}e^{-\eta \epsilon}}, \label{update_rule}
\end{eqnarray}
where $w_{D,t} = 1-w_{P,t}$ stands for the dummy's weight. 
The pseudo-code for SafePredict is given in Algorithm 2.

\begin{algorithm}[!htb]
\caption{ \textbf{SafePredict}\\
\textit{Base predictor}: $P$; ~\textit{Initial weight}: $w_{P,1} \in (0,1)$ \\
\textit{Learning rate}: $\eta > 0$; ~ \textit{Target error rate}: $\epsilon \in (0,1)$}
\begin{algorithmic}[1]
\For {each $t=1,2,\ldots$}
\State Predict with probability $w_{P,t}$, refuse otherwise, i.e.
\[\hat{y}_{t} =   \left\{
\begin{array}{ll}
      \hat{y}_{P,t} & ~~~\textrm{with prob. } w_{P,t} \\
      \varnothing & ~~~\textrm{otherwise} \\
\end{array} 
\right. \]
\State Update the prediction probability:
\begin{eqnarray}
w_{P,t+1} & = & \frac{w_{P,t}e^{-\eta l_{P,t}}}{w_{P,t}e^{-\eta l_{P,t}} + w_{D,t}e^{-\eta \epsilon}}, \nonumber \end{eqnarray}
\EndFor
\end{algorithmic}
\end{algorithm}

Note that Algorithm 2 requires no assumptions (such as i.i.d. assumptions) about the input data. The algorithm takes  \emph{only} the loss values as input.

\subsection{Validity}
To bound the average loss of SafePredict one might be inclined to simply apply the bound presented in Theorem 2.1 with $N=2$, $P_1=P$,  $P_2=D$, and the optimal learning rate to minimize the RHS of eq. (\ref{eq:b1}) for $i=2$, which results in
\begin{eqnarray}
\frac{L_{P,T}^*}{ T^*} & \leq & \epsilon + \frac{1}{T^*}\sqrt{\frac{T\log (1/ w_{D,1})}{2}}.
\label{eq:badbound}
\end{eqnarray}
The problem with the bound in eq. (\ref{eq:badbound}) is that if the prediction probability decreases too quickly (i.e. for $T^*\ll\sqrt{T}$) the rightmost term can easily become vacuous, i.e. the excess error does not vanish. In particular, the bound is not sufficient to guarantee the validity of the algorithm for the case in which  $T^*$ the number of predictions made approaches infinity slower than $\sqrt{T}$ {where T is the number of data points presented}, i.e. $T^* = \omega \left(1\right)$ and  $T^* = O\left(\sqrt{T}\right)$. 
 
In the following theorem, we present a refined bound for SafePredict that suggests a learning rate that decreases with the variance of the number of predictions, i.e. $V^*$, and guarantees the validity of our algorithm.

\begin{theorem} For any $P$, $\eta>0$, $\epsilon < 1/2$, and $0< w_{P,1} < 1$, SafePredict satisfies 
\begin{eqnarray}
\frac{ L_{P,T}^* }{T^*} & \leq & \epsilon - \frac{\log \left(w_{D,1}\right) }{\eta T^*} + \frac{\left(1 - \epsilon\right)^2\eta V^*}{T^*}. \label{valid_guarantee}
\end{eqnarray}
Consequently, by choosing the learning rate $\eta$ to minimize the RHS of this bound, we {get}
\begin{eqnarray}
\frac{L_{P,T}^*}{T^*} & \leq & \epsilon + \left(1 - \epsilon\right)\frac{2\sqrt{\log \left(1/w_{D,1}\right) V^*}}{T^*}  \label{valid_guarantee2}
\end{eqnarray}
for $\eta^* =\frac{\sqrt{{\log \left(1/w_{D,1}\right)}/{V^*}}}{1-\epsilon}$.
\end{theorem}

Before proving the statement given in the theorem, we define auxiliary quantities called ``mix-loss'' and ``mixability gap'', and present two important results from \cite{de2014follow} about EWAF in terms of these quantities. 

In the expert advice framework of Section 2.1, the \textit{mix-loss} $m_t$ is an alternative way of averaging the expert losses $l_{P_i,t},~\forall~i$ using the weights $w_{P_i,t},~\forall~i$. Formally, {\textit{mix-loss}} is defined as 
 \[m_t = -\frac{1}{\eta} \log \left(\sum_{i=1}^N w_{P_i,t} e^{-\eta l_{P_i,t}}\right)\]
 or equivalently $e^{-\eta m_t} = \sum_{i=1}^N w_{P_i,t} e^{-\eta l_{P_i,t}}$. Additionally, we denote the cumulative mix-loss as $M_T = \sum_{t=1}^Tm_t$.
 
One can bound the cumulative mix-loss in terms of the losses of the individual experts by the following lemma.
\begin{lemma} (\cite{de2014follow}, Lemma 1)
For the learning rate $\eta > 0$, and the initial weights $\left(w_{P_1,1},\ldots, w_{P_N,1}\right)$, the cumulative mix-loss of the EWAF satisfies the following inequality
\begin{eqnarray}
M_T & \leq & L_{P_i,T} - \frac{\log w_{P_i,1}}{\eta} 
\end{eqnarray}
for all $i= 1,\ldots, N$.
\end{lemma}

The bound on the cumulative mix-loss can be  used to bound the cumulative expected loss by defining the \emph{mixability gap} as the difference between these two type of averages and bounding them with classical concentration inequalities. In particular, the mixability gap, $\delta_t = l_t - m_t$, can be bounded using the following lemma. 

\begin{lemma} (\cite{de2014follow}, Lemma 4)
The difference between the expected and mix losses obtained by EWAF algorithm is bounded by
\[\delta_t = l_t - m_t \leq \eta \sum_{i} w_{P_i,t} \left(l_t - l_{P_i,t}\right)^2 \] 
for all $t$.
\end{lemma}

Armed with these two lemmas, we can continue with the proof of Theorem 3.1
\begin{proof}[Proof of Theorem 3.1]
By definition of $\delta_t$ we have 
\begin{eqnarray}
L_T & = & M_T + \sum_{t=1}^T \delta_t. \nonumber
\end{eqnarray}
Then by applying Lemma 3.2. with $P_i = D$, we get
\begin{eqnarray}
\sum_{t=1}^T l_t & \leq & \epsilon T -\frac{\log \left(w_{D,1}\right) }{\eta} + \sum_{t=1}^T \delta_t. \label{first_step}
\end{eqnarray}

Next, we plug in the {definition} of the expected loss $l_t = w_{P,t} l_{P,t} + w_{D,t}\epsilon$, and divide both sides by the expected number of predictions, 
\begin{eqnarray}
\sum_{t=1}^T w_{P,t} l_{P,t}  & \leq & \sum_{t=1}^T (1-w_{D,t})\epsilon  -\frac{\log \left(w_{D,1}\right) }{\eta} + \sum_{t=1}^T \delta_t \nonumber \\
{L_{P,T}^*}  & \leq & \epsilon T^*   -\frac{\log \left(w_{D,1}\right) }{\eta } +{\sum_{t=1}^T \delta_t} \nonumber \\ 
\frac{L_{P,T}^*}{T^*}  & \leq & \epsilon -\frac{\log \left(w_{D,1}\right) }{\eta T^* } +\frac{1}{T^*}{\sum_{t=1}^T \delta_t}.
\label{last_step}
\end{eqnarray}

Finally, we obtain the desired result by bounding $\delta_t$:
\begin{align}
\delta_t ~ & \leq  ~ \eta\left( w_{P,t} \left(l_t - l_{P,t}\right)^2 + w_{D,t} \left(l_t - \epsilon \right)^2\right) \label{a1} \\
& = ~ {\eta}\left( w_{P,t} w_{D,t}^2  + w_{D,t} w_{P,t}^2 \right)\left(l_{P,t} - \epsilon \right)^2 \label{b1} \\
& = ~ {\eta} w_{P,t} w_{D,t}\left( l_{P,t} - \epsilon\right)^2  \label{c1} \\
& \leq ~ {\eta} w_{P,t} w_{D,t}\left( 1 - \epsilon\right)^2 \label{mix_gap_bound}
\end{align}
(\ref{a1}) follows from Lemma 3.3, while (\ref{b1}) and (\ref{c1}) follow from the definition of $l_t$ and $w_{D,t} + w_{P,t} = 1$, respectively.

Then, by summing up both sides of eq. (\ref{mix_gap_bound}) over $t$, we get
\begin{eqnarray}
\sum_{t=1}^T\delta_t & \leq & \frac{\eta V^*}{\left(1-\epsilon\right)^2}. \label{mix_gap_cum}
\end{eqnarray}

The desired result for the first part of the theorem, eq. (\ref{valid_guarantee}), follows by plugging eq. (\ref{mix_gap_cum}) into eq. (\ref{last_step}). Finally, the learning rate $\eta$ that minimizes the RHS of eq. (\ref{valid_guarantee}) can be found by basic calculus, and by plugging in this optimal learning rate we obtain {eq.} (\ref{valid_guarantee2}).
\end{proof}

Note that the essential difference between the bounds in eq. (\ref{eq:badbound}) and (\ref{valid_guarantee2}) lies in the choice of the learning rate $\eta$. While the optimal learning rate for the conventional use of EWAF is on the order of $1/\sqrt{T}$, Theorem 3.1 suggests to choose $\eta$ to be on the order of $1/\sqrt{V^*}$, which decreases much slower than $1/\sqrt{T}$,  
and leads to a sufficient condition for the validity of our algorithm.

Theorem 3.1 implies the excess error rate decreases with a rate ${\sqrt{V^*}}/{T^*}$, i.e. $L_{P,t}^*/T^* - \epsilon = O\left({\sqrt{V^*}}/{T^*}\right)$. Since $V^*$ is always less than or equal to $T^*$, this rate is bounded by
\[\frac{1}{T^*} \leq \frac{\sqrt{V^*}}{T^*} \leq \frac{1}{\sqrt{T^*}}.\]
Therefore our result implies that as more predictions are made (i.e. $T^*$ increases) the bound becomes tighter and guarantees the validity of SafePredict algorithm. 
\begin{corollary}
For any $P$, $\epsilon < 1/2$, $0< w_{P,1} < 1$, if one choose the learning rate $\eta$ in the order of $1/\sqrt{V^*}$, i.e. $\eta = \Theta\left(1/\sqrt{V^*}\right)$,  SafePredict is guaranteed to be valid.
\end{corollary}

Unfortunately, selecting a learning rate $\eta$ that depends on $V^*$ is infeasible because $V^*$ is unknown a priori. In Section 3.3 we describe a practical method for choosing a learning rate without knowledge of the expected number of refusals while still satisfying eq. (\ref{valid_guarantee2}) up to a constant factor. Having shown the validity of our algorithm, we now move to the efficiency.

\subsection{Efficiency}
In this section, we study the efficiency of the SafePredict meta-algorithm given in Algorithm 2. In contrast with validity which we addressed without requiring any assumptions on the base predictor, we must characterize the efficiency of a given algorithm with respect to the performance of the base predictor $P$. We show that SafePredict leads to efficient predictions if the base predictor $P$ has an asymptotic error rate smaller than the target error rate $\epsilon$, i.e. $\lim_{t\rightarrow \infty} L_{P,t}/t < \epsilon.$

The following lemma lower and upper bounds the probability of making a prediction at time $t+1$, i.e. $w_{P,t+1}$, in terms of the cumulative loss of the base predictor $P$. 

\begin{lemma} The probability of making a prediction at time $t+1$ for SafePredict, namely $w_{P,t+1}$, satisfies the following 
\begin{align}
1 - {\frac{w_{D,1}}{w_{P,1}} e^{ \eta  \left(L_{P,t} - \epsilon t\right)} } \leq w_{P,t+1}  \leq  {\frac{w_{P,1}}{w_{D,1}} e^{ \eta  \left(\epsilon t - L_{P,t}\right)} }. \label{dummy_bound}
\end{align}
\end{lemma}

\begin{proof}
To prove the upper bound in eq. (\ref{dummy_bound}) we first note that the update rule given in eq. (\ref{update_rule}) can be written in terms of the mix-loss, and by induction we can obtain the prediction probability at time $t+1$ in terms of the cumulative losses $M_t$ and $L_{P,t}$ (see (\ref{a2}) and (\ref{b2}) below), i.e.    
\begin{align}
w_{P,t+1}~ & = ~ \frac{w_{P,t}e^{-\eta l_{P,t}}}{e^{-\eta m_t}} ~  = ~  w_{P,t}e^{\eta \left(m_t - l_{P,t}\right)} \label{a2} \\
 ~ & =  ~  w_{P,1}e^{\eta \left(M_t - L_{P,t}\right)}  \label{b2} \\
 ~ & \leq ~ w_{P,1}e^{\eta \left(L_{D,t} - \log(w_{D,1})/\eta - L_{P,t}\right)} \label{c2} \\ 
& \leq ~  \frac{w_{P,1}}{w_{D,1}}e^{\eta \left(\epsilon t - L_{P,t}\right)}. \label{d2}
\end{align}

Next, (\ref{c2}) and (\ref{d2}) follows by applying Lemma 3.2 with $P_i = D$ to bound the cumulative mix loss in terms of the cumulative loss of the dummy, $L_{D,t} = \epsilon t$.

For the lower bound, we apply the same argument for $w_{D,t+1} = 1-w_{P,t+1}$. \end{proof}
Note that {Lemma 3.4 implies that} the probability of making a prediction decreases exponentially fast if the cumulative loss of the base predictor is increasing faster than the target error rate $\epsilon$. Next, we exploit this fact to show that if the base predictor satisfies the desired error requirement (i.e. its average loss is already below $\epsilon$), {our algorithm achieves a finite expected number of refusals.}

\begin{theorem}
If $ L_{P,T} / T < \epsilon$ and $\eta T \rightarrow \infty$ in the limit $T\rightarrow \infty$, then the expected number of refusals made by SafePredict is finite, i.e. 
\begin{eqnarray}
\lim_{T\rightarrow\infty}T-T^* & = & \sum_{t=1}^\infty w_{D,t} ~ < ~ \infty. \label{finite}
\end{eqnarray}
\end{theorem}
\begin{proof} 
We first define $\epsilon' = \limsup_{T\rightarrow \infty} L_{P,T} / T $. Since $\epsilon' < \epsilon$, there exists a $t_0 < \infty $ such that for all $t > t_0$:
\begin{eqnarray}\frac{L_{P,t}}{t} & \leq & \epsilon' + \frac{\epsilon - \epsilon'}{2}. \label{lim}\end{eqnarray}
Then, 

\begin{align}
\sum_{t=1}^\infty w_{D,t} ~ & = ~ \sum_{t=1}^{t_0} w_{D,t} + \sum_{t=t_0+1}^\infty w_{D,t}  \\
& \leq ~ t_0 +  \sum_{t=t_0+1}^\infty \left( 1-w_{P,t} \right) \label{a3} \\
& \leq  t_0 + \frac{1- w_{P,1}}{w_{P,1}} \sum_{t=t_0+1}^\infty  e^{\eta t \left(L_{P,t}/t - \epsilon\right)} \label{b3} \\
& \leq  t_0 +  \lim_{t'\rightarrow \infty} \frac{w_{D,1}}{w_{P,1}} \sum_{t=t_0+1}^{t'} e^{\eta t \left(\epsilon' - \epsilon\right)/2}  \label{c3} \\
& \leq  t_0 +  \lim_{t'\rightarrow \infty} \frac{w_{D,1}}{w_{P,1}} \frac{e^{\eta (t_0+1) \left(\epsilon' - \epsilon\right)/2} - e^{\eta t' \left(\epsilon' - \epsilon\right)/2}}{1-e^{\eta \left(\epsilon' - \epsilon\right)/2}} ~  \label{d3} \\
& \leq  t_0 +  \frac{w_{D,1}}{w_{P,1}} \frac{e^{\eta (t_0+1) \left(\epsilon' - \epsilon\right)/2}}{1-e^{\eta \left(\epsilon' - \epsilon\right)/2}} ~ < ~ \infty, \label{e3}
\end{align}

where  (\ref{a3}) follows from the fact that $w_{D,t} \leq 1$, (\ref{b3}) follows from Lemma 3.4, (\ref{c3}) follows from eq. (\ref{lim}), (\ref{d3}) follows from the sum of a geometric series and from the fact $e^{\eta \left(\epsilon' - \epsilon\right)/2} < 1$, finally (\ref{e3}) follows from the hypothesis $\eta T \rightarrow \infty$.
\end{proof}
 
Theorem 3.5. shows that SafePredict is efficient if the base predictor is already valid. Furthermore, the theorem shows the expected number of refusals is finite even over an infinite data stream. The following corollary strengthens the operational insight from the expected value to an \emph{almost sure} statement.

\begin{corollary}
If $L_{P,t}/t < \epsilon$ in the limit $t\rightarrow \infty$ and $\eta = \omega \left(T^{-1}\right)$, i.e. $\eta T \rightarrow \infty$, then SafePredict leads to an efficient predictor for $P$ and $\epsilon$, i.e. 
\[\liminf_{T \rightarrow \infty} \rho_T 
= \liminf_{T \rightarrow \infty} 1 - \frac{\sum_{t=1}^T w_{D,t}}{T}  = 1.\]

Furthermore, there exists a finite $T_0$, such that the number of refusals are less than $T_0$ almost surely.
\end{corollary}
\begin{proof}
The first part of the corollary is a direct consequence of  Theorem 3.5, and the second part follows from the Borel-Cantelli Lemma, see e.g. Theorem 3.9 from \cite{krickeberg1965probability}.
\end{proof}

Note that Corollary 3.5.1 depends on the assumption that the predictor $P$ is able to (eventually) predict accurately enough to obtain a valid algorithm. Yet, if the predictor is highly inaccurate, i.e. $l_{P,t} > \epsilon$ with high probability, then SafePredict meta-algorithm will refuse almost all the time, as it should. The meta-algorithm does not need to know in advance how well the base predictor will behave to achieve this desirable outcome. Further, we can make SafePredict adaptive when the base predictor sometimes is valid and sometimes is not, as we show in Section 4.

\subsection{Choosing the Learning Rate}

Corollary 3.3.1 and 3.5.1 establish the validity and efficiency of SafePredict once the learning rate $\eta$ is chosen to be on the order of $1/\sqrt{V^*}$, 
but $V^*$ cannot be known a priori. One classical method of addressing the issue of estimating unknown quantities in an online scenario is known as the \emph{doubling trick} (cf. Chapter 2.3 of \cite{cesa2006prediction} or Theorem 7.7 of \cite{mohri2012foundations}). In this subsection, we derive a  validity bound using this trick.

In general, the doubling trick starts with an initial estimate of the unknown quantity and compares it with the observed value of this quantity at each time step. When the measured value exceeds the estimate, the estimate is doubled and the algorithm is reset. 

We use the doubling trick to choose $\eta$. Specifically, we want to estimate $V^*$ for a fixed horizon $T$. We start with an initial estimate of $V_{est}$ (typically $V_{est}= 1$) and a running sum $V_{sum} = 0$. Then we invoke Theorem 3.1. by replacing $V^*$ with $V_{est} = 1$, i.e. $\eta = \sqrt{{\log \left(1/w_{D,1}\right)}}/{\left(1-\epsilon\right) }$ and run Algorithm 2. At each time $t$ we update $V_{sum}$ by incrementing it by $w_{P,t+1}w_{D,t+1}$ and check if it exceeds the current estimate. If it does, we do the following:
\begin{compactenum}
\item Double the estimated value $V_{est} \gets 2 V_{est}$.
\item Update the learning rate according to the new $V_{est}$, i.e. $\eta \gets \eta / \sqrt{2}$.
\item Reset the prediction probability $w_{P,t+1} \gets w_{P,1}$.
\item Reset the running sum $V_{sum} \gets 0$. 
\end{compactenum}

The complete pseudo-code for the modified algorithm is given in Algorithm 3.
\begin{algorithm}[!htb]
\caption{ \textbf{SafePredict with Doubling Trick}\\
\textit{Base predictor}: $P$; ~\textit{Initial weight}: $w_{P,1} \in (0,1)$ \\
\textit{Target error rate}: $\epsilon \in (0,1)$}
\begin{algorithmic}[1]
\State Initialize $t = 1$
\For {each $k=1,2,\ldots$}
\State Reset $w_{P,t} = w_{P,1}$, ~ $V_{sum} = 0$, ~ and 
\[\eta = \sqrt{\log(1/w_{D,1})/\left(1-\epsilon \right)^2/2^{k}}\]
\While {$V_{sum} \leq 2^k$}
\State Predict with probability $w_{P,t}$, refuse otherwise,
\[\hat{y}_{t} =   \left\{
\begin{array}{ll}
      \hat{y}_{P,t} & ~~~\textrm{with prob. $w_{P,t}$}  \\
      \varnothing & ~~~\textrm{otherwise} \\
\end{array} 
\right. \]
\State Update the prediction probability:
\begin{eqnarray}
w_{P,t+1} & = & \frac{w_{P,t}e^{-\eta l_{P,t}}}{w_{P,t}e^{-\eta l_{P,t}} + w_{D,t}e^{-\eta \epsilon}}, \nonumber \end{eqnarray}
\State Compute $V_{sum} \gets V_{sum} + w_{P,t+1}w_{D,t+1}$
\State Increment $t$ by $1$, i.e. $t \gets t+1$
\EndWhile
\EndFor
\end{algorithmic}
\end{algorithm}

In the following theorem, we show that for SafePredict with doubling trick, given in Algorithm 3, the validity bound  presented in Theorem 3.1 increases only by a constant multiplicative factor of $\sqrt{2}/(\sqrt{2}-1)$.

\begin{theorem} For any $P$, $\epsilon < 1/2$, and $0< w_{P,1} < 1$, SafePredict with the doubling trick given in Algorithm 3 satisfies 
\begin{eqnarray}
\frac{L_{P,T}^* }{T^*} & \leq & \epsilon + \left(1 - \epsilon\right)\frac{2 \sqrt{2}}{\sqrt{2}-1}\frac{\sqrt{\log \left(1/w_{D,1}\right)V^*}}{T^*}. \label{doubling_trick}  \end{eqnarray}
\end{theorem}

\begin{proof}
Denote the time for the $K^{th}$ reset of the algorithm with $T_{K}$, i.e. $T_K$ is the largest $\tau$ such that $\sum_{t=1}^{\tau} w_{P,t} w_{D,t} \leq 2^{K}$. Additionally, assume $T_0 = 0$ and $K^*$ is the integer that satisfies 
$T_{K^*-1} < T \leq T_{K^*} $.

Then, we can rewrite the sum 
\begin{eqnarray}
L_{P,T}^* - \epsilon T^* & = & \sum_{t=1}^T w_{P,t}\left(l_{P,t} - \epsilon\right) \nonumber \\
 &   \leq & \sum_{K=1}^{K^*}\sum_{t=T_{K-1}+1}^{T_K}w_{P,t} \left(l_{P.t} - \epsilon\right). \label{step1}
\end{eqnarray}

Next, we can bound each summand in the right hand-side using eq. (\ref{valid_guarantee2}) from Theorem 3.1, i.e. for all $K$ we plug the estimated $V_{est} = 2^K$ value instead of $V^*$ in the learning rate and the corresponding bound,
\begin{align}
\sum_{t=T_{K-1}+1}^{T_{K}}w_{P,t} \left(l_{P,t} - \epsilon\right)  & \leq  ~ \left(1-\epsilon\right) 2\sqrt{2^{K} \log \left({1}/{w_{D,1}}\right)}. \label{step2} 
\end{align}
Combining eq. (\ref{step1}) and (\ref{step2}), we obtain
\begin{align}
L_{P,T}^* - \epsilon T^* & \leq ~\left(1-\epsilon\right)2 \sum_{K=1}^{K^*}\sqrt{2^{K} \log \left(1/w_{D,1}\right)}  \nonumber \\
& = ~ \left(1-\epsilon\right)2 \sqrt{2^{K^*} \log \left(1/w_{D,1}\right)}\sum_{K=1}^{K^*} \sqrt{2^{K-K^*}} \nonumber \\
& \leq ~ \left(1-\epsilon\right) \frac{2\sqrt{2}}{\sqrt{2}-1} \sqrt{2^{K^*} \log \left(1/w_{D,1}\right)} \label{a4} \\
& \leq ~ \left(1-\epsilon\right) \frac{2\sqrt{2}}{\sqrt{2}-1} \sqrt{\log \left(1/w_{D,1}\right) V^*} \label{b4} 
\end{align}
where (\ref{a4}) follows from the sum of a geometric series and (\ref{b4}) follows from the definition of $K^*$.

Finally, the desired result is obtained by dividing both sides by $T^*$.
\end{proof}

Theorem 3.5 shows that the penalty paid for not knowing the optimal $\eta$ is a constant factor in the validity bound. For efficiency, our results from Theorem 3.5 and Corollary 3.5.1 still apply, and only finitely many refusals occur if the error rate of the base predictor is below $\epsilon$. To see this, first, note that we simply run SafePredict by restarting the weights for increasingly sized blocks of data points. In each block our estimate of the learning rate is off by at most a constant factor of $2$ from the suggested learning rate in Theorem 3.1. Therefore, the condition on the learning rate of Theorem 3.5 (i.e. $\eta T \rightarrow \infty$) is guaranteed to be satisfied and thus efficiency is guaranteed.

\section{Adaptive SafePredict}
We now consider the practical scenario in which the error rate of the base predictor changes over time. Our meta-algorithm should reduce the probability of making a prediction when the base predictor suffers a large error and predict more often when the base predictor does well. Inspired by the Fixed Share algorithm \cite{herbster1998tracking}, we introduce Adaptive SafePredict, a weight-shifting extension to SafePredict. Adaptive SafePredict tracks changes in the error rate of the base predictor while preserving validity, thus improving efficiency even when there are periods of poor predictions.  

The base predictor can have a non-constant error rate for a variety of reasons. For example, most predictors have a high error rate at the beginning of the prediction task and the error rate decreases as they see more examples and learn from the mistakes. Alternatively, the underlying data distribution may abruptly change and thus the performance of the base predictor degrades significantly till the predictor learns the new data distribution. Such scenarios might lead to long sequences of bad predictions for the base algorithm $P$, and force the prediction probability to tend to zero. To make sure the prediction probability does not decrease too quickly due to a long sequence of bad predictions, we shift a small portion of the dummy's  ``weight'' (refusal probability) towards the base predictor. This weight shift allows Adaptive SafePredict to quickly recover when the base predictor performs well again.

As an example, suppose $P$ has a constant error rate $3\epsilon$ for $t \leq T/2$ and $\epsilon/2$ for $T/2 <t\leq T$. Ideally, we would like  $w_{P,t} = 0$ for $t\leq T/2$ and $w_{P,t}=1$ for $t> T/2$ to achieve efficiency $\rho_T = 1/2$. However, Lemma 3.4 implies the probability of making a prediction decreases quickly till time $T/2$ and reaches $\approx e^{-\eta \epsilon T}$. Even though it starts to increase afterwards, at time $T$ the probability of making a prediction is still bounded by $\approx e^{\eta \left(\epsilon T - L_{P,T}\right) } = e^{-\eta \epsilon 3T /4}$, which is far below the ideal probability of $1$.

In the literature, there are various ways of making the generic EWAF adaptive against such changing environments, see e.g. Chapter 5.2 of \cite{cesa2006prediction}.
A particularly elegant and popular way among these techniques is called "sharing the weights" which leads to the \textit{fixed share algorithm} \cite{herbster1998tracking}. The fixed share algorithm simply adds a mixing step to the weight update rule in the EWAF. In particular, following the notation from Section 2.1, the fixed share update rule becomes
\[w_{P_i,t+1} = \frac{\alpha}{N} + \left(1-\alpha\right)\frac{w_{P_i,t}e^{-\eta l_{i,t}}}{\sum_{j=1}^Nw_{P_j,t}e^{-\eta l_{j,t}}}\]
for some $\alpha \in [0,1)$. 
The mixing step ensures that no weight falls below a predefined value $\alpha/N$, and guarantees a sub-linear regret against roughly $\alpha T$ abrupt changes in the underlying statistics (e.g. the error rate). For a detailed analysis of the fixed share algorithm, please see \cite{adamskiy2012closer,herbster1998tracking}.

In this section, we apply a similar idea to our problem and propose an adaptive version of the SafePredict. We modify the EWAF update rule eq. (\ref{update_rule}) to guarantee that the prediction probabilities do not get too small or too large, so that they can track the changes in the error rate efficiently while preserving the validity guarantees established in Section 3.1. To constrain the prediction probability at time $t+1$ to lie within an arbitrary interval $\alpha_t \leq w_{P,t+1} \leq \beta_t$, we  modify our update rule as follows 
\begin{eqnarray}
w_{P,t+1} & = & \alpha_t + \left(\beta_t-\alpha_t \right)\frac{w_{P,t}e^{-\eta l_{P,t}}}{w_{P,t}e^{-\eta l_{P,t}} + w_{D,t}e^{-\eta \epsilon}}.
\label{ws_update_rule}
\end{eqnarray}
We call $\alpha_t$ and $\beta_t$ the \textit{adaptivity parameters} and let them change with time for the sake of generality. In Algorithm 4, we give the pseudo-code for Adaptive SafePredict with adaptivity parameters $\alpha_t$ and $\beta_t$.
\begin{algorithm}[!htb]
\caption{ \textbf{Adaptive SafePredict}\\
\textit{Base predictor}: $P$; ~\textit{Initial weight}: $w_{P,1} \in (0,1)$ \\
\textit{Learning rate}: $\eta > 0$; ~ \textit{Target error rate}: $\epsilon \in (0,1)$\\
\textit{Min. prediction prob.}: $\alpha_1,\ldots, \alpha_T\in [0,1)$ \\  \textit{Max. prediction prob.}: $\beta_1, \ldots, \beta_T \in (0,1]$}
\begin{algorithmic}[1]
\For {each $t=1,2,\ldots$}
\State Predict with probability $w_{P,t}$, refuse otherwise, i.e.
\[\hat{y}_{t} =   \left\{
\begin{array}{ll}
      \hat{y}_{P,t} & ~~~\textrm{with prob. } w_{P,t} \\
      \varnothing & ~~~\textrm{otherwise} \\
\end{array} 
\right. \]
\State Update the prediction probability:
\begin{eqnarray}
w_{P,t+1} & = & \alpha_t + \left(\beta_t-\alpha_t\right)\frac{
w_{P,t}e^{-\eta l_{P,t}} 
}{
w_{P,t}e^{-\eta l_{P,t}} +  w_{D,t}
e^{-\eta\epsilon}} \nonumber
\end{eqnarray}
\EndFor
\end{algorithmic}
\end{algorithm}

In the rest of this section, we find a good setting for the adaptivity parameters that provides resilience against changes while preserving validity. In particular, we first extend our analysis from Section 3.1 and observe that validity is preserved as long as $\alpha_t$ is on the order of $1/T$, irrespective of the choice of $\beta_t$. By noting $w_{P,t+1}$ is an increasing function of $\beta_t$, we choose  $\beta_t = 1$ to maximize the efficiency of this adaptive algorithm (Alg. $4$) under this restriction on $\alpha_t$. In particular, we derive an upper bound to the probability of refusal that depends only on the loss sequence of $P$ starting from an arbitrary time point $t_0$, i.e. $L_{P,t,t_0}$ instead of $L_{P,t}$, that implies a boost in  efficiency if $P$ starts to make better predictions after $t_0$.

\subsection{Validity}
To quantify the effect of the adaptivity parameters $\alpha_t, \beta_t$ on our validity guarantees, we first show that Adaptive SafePredict is equivalent to the usual EWAF with a larger set of experts. In particular, we consider a virtual ensemble, where each expert in the ensemble chooses to follow either the dummy predictor $D$, or the base predictor $P$. In the following lemma, we show that for an appropriately chosen set of initial weights, the EWAF on this virtual ensemble (i.e. Algorithm $1$) becomes equivalent to Adaptive SafePredict as outlined in Algorithm $4$. Next, we use this equivalence to extend our analysis from Section $3$.

\begin{lemma}{(Equivalence lemma)}  \label{equivalence}
Suppose we have a base predictor $P$. Consider an ensemble of $2^T$ experts, $\mathcal{P} = \{P_0,P_1,\ldots,P_{2^T-1}\}$, defined as follows:
\begin{compactitem}
\item Denote the $t^{th}$ bit of the binary expansion of integer $i$ with $b_{i,t}$, and define the notation $\bar{x} = 1- x$.
\item Fix the predictions and the losses of each expert $P_i$ as follows:
\[ \hat{y}_{P_i,t} =   \left\{
\begin{array}{ll}
      \varnothing & b_{i,t} = 0  \\
      \hat{y}_{P,t} & b_{i,t} = 1  \\
\end{array} 
\right. \textrm{ and } l_{P_i,t} =   \left\{
\begin{array}{ll}
      \epsilon &  b_{i,t} = 0  \\
      l_{P,t} & b_{i,t} = 1 \\
\end{array} 
\right. .\]
\item Set the initial weights for each expert as
\begin{align*}
w_{P_i,1} ~ =  &   w_{P,1}^{b_{i,1}} w_{D,1}^{\bar{b}_{i,1}} \ldots \\ & \quad \prod_{t=1}^{T-1}  \alpha_t^{\bar{b}_{i,t}{b}_{i,t+1}}\bar{\alpha}_t^{\bar{b}_{i,t}\bar{b}_{i,t+1}}  \beta_t^{{b}_{i,t}{b}_{i,t+1}}\bar{\beta}_t^{{b}_{i,t}\bar{b}_{i,t+1}}.
\end{align*}
\end{compactitem}

Then the EWAF algorithm (Alg. 1) using the expert ensemble $\mathcal{P}$ with the learning rate $\eta$ is equivalent to Adaptive SafePredict (Alg. 4) using the base predictor $P$, in terms of the prediction probability  
\[w_{P,t} = \sum_{i: b_{i,t}=1} w_{P_i,t}.\]
\end{lemma}

\begin{proof} 
The proof is in Supplementary Material.
\end{proof}

Because Lemma 4.1 reduces the adaptive algorithm to an instance of EWAF, we can obtain the following validity guarantee by modifying the proof of Theorem 3.1.

\begin{corollary} For any $P$, $\eta>0$, $\epsilon < 1/2$, $0< w_{P,1} < 1$, and $0 \leq \alpha_t, \beta_t, \leq 1$, $\forall t$ Adaptive SafePredict meta-algorithm given in Algorithm 4 satisfies 
\begin{eqnarray}
\frac{ L_{P,T}^* }{T^*} & \leq & \epsilon - \frac{\log \left(w_{D,1} \Delta_T\right) }{\eta T^*} + \frac{\left(1 - \epsilon\right)^2\eta V^*}{T^*},
\end{eqnarray}
where  $\Delta_T$ is defined as $\prod_{t=1}^{T-1} \left(1-\alpha_t\right)$.

By choosing the learning rate $\eta$ to minimize the RHS of this bound, we {get}
\begin{eqnarray}
\frac{L_{P,T}^*}{T^*} & \leq & \epsilon + \left(1 - \epsilon\right)\frac{2\sqrt{  \log \left(1/\left(w_{D,1}\Delta_T\right)\right) V^*}}{T^*}  \label{leak_bound}
\end{eqnarray}
for $\eta^* =\frac{\sqrt{{\log \left(1/\left(w_{D,1}\Delta_T\right)\right)}/{V^*}}}{1-\epsilon} $.
\end{corollary}

\begin{proof}
The proof of the corollary follows from exactly the same steps given in the proof of Theorem 3.1, except instead of choosing $P_i = D$ while applying Lemma 3.2 in eq. (\ref{first_step}), we choose $P_i = P_0$ from the equivalent virtual ensemble given in Lemma 4.1. Note that $P_0$ always refuses and essentially identical to $D$, except its initial weight is 

\vspace{3pt}
$\displaystyle w_{P_0,1} ~=~ w_{D,1} \prod_{t=1}^{T-1} \left(1-\alpha_t\right) ~=~ w_{D,1}\Delta_T.$
\end{proof}

\begin{corollary}
For $\alpha_t = \alpha < 1/2~ \forall t$, the validity bound given in eq. (\ref{leak_bound}) becomes
\begin{eqnarray}
\frac{L_{P,T}^*}{ T^*}   & \leq & \epsilon + \left(1 - \epsilon\right)\frac{2\sqrt{ V^*\left(\log \left( 1/{w_{D,1}}\right) +T \alpha + T\alpha^2\right)}}{T^*}. \nonumber
\end{eqnarray}

Thus setting $\alpha = O\left(1/T\right)$ is a sufficient condition to guarantee the validity of Adaptive SafePredict given in Alg. 4.
\end{corollary}
\begin{proof}
Proof directly follows from setting $\alpha_t = \alpha$ in eq. (\ref{leak_bound}) and using Taylor series expansion of $\Delta_T$,

\vspace{3pt}
$\displaystyle \log \left(\Delta_T\right) ~ = ~ (T-1) \log(1-\alpha) ~ \geq ~ -T \left(\alpha + \alpha^2\right).$
\end{proof}

Corollary 4.1.2 implies that by choosing $\alpha_t = \alpha = O\left(1/T\right)$, we can preserve the same convergence rate for the excess error rate from Theorem 3.1, i.e. $L_{P,T}^*/T^* - \epsilon = O\left(\sqrt{V^*}/T^*\right)$. In other words, as long as $\alpha$ is small, the whole effect of the weight shifting on our validity bound can be interpreted as starting with a smaller initial refusal probability, by reducing  $w_{D,1}$ by a multiplicative factor of $\Delta_T \approx e^{\alpha T}$ which is essentially a constant for $\alpha = O\left(1/T\right)$.

Furthermore, note that the bound given in eq. (\ref{leak_bound}) depends only on $\alpha_t$ and is totally agnostic to the choice of $\beta_t$. Therefore, once $\alpha_t$ values are chosen to preserve the validity, i.e. on the order of $1/T$, we are free to choose $\beta_t$ to maximize the prediction probability, and therefore efficiency, by choosing $\beta_t = 1$ at all time points $t$. In the next subsection, we analyze the efficiency of this special case and argue it provides resilience against changes in the data distribution. 

\subsection{Weight-Shifting SafePredict}
In Section 3.2, we showed that the 	prediction probability of SafePredict increases or decreases exponentially quickly depending on the cumulative loss of the base predictor $P$. As mentioned in the beginning of Section 4, this may cause a high refusal rate when the performance of the base predictor deteriorates abruptly. In this section, we exploit a special case of Adaptive SafePredict to show that it provides resilience against changes in the data distribution while preserving the validity of the algorithm. This special case, obtained for $\alpha_t = \alpha = \Theta\left(1/T\right)$ and $\beta_t = 1$, is called \textit{Weight-Shifting SafePredict}.

As motivated by  Corollary 4.1.2, we first set $\alpha_t$ for all $t$ equal to some constant $ \alpha$, where $\alpha$ is on the order of $1/T$. This will  guarantee validity. Next, we choose to maximize $w_{P,t+1}$ over $\beta_t$ in order to maximize the efficiency. Therefore, from eq. (\ref{ws_update_rule}) we obtain $\beta_t = 1,~\forall~t$. Note that this particular choice of adaptivity parameters simplifies the update rule given in eq. (\ref{ws_update_rule}) to
\begin{eqnarray}
w_{P,t+1} & = & \alpha + \left(1-\alpha \right)\frac{w_{P,t}e^{-\eta l_{P,t}}}{w_{P,t}e^{-\eta l_{P,t}} + w_{D,t}e^{-\eta \epsilon}}. \label{ws_simplified}
\end{eqnarray}
An intuitive way of looking at this rule is that at each time point we use the EWAF rule first and then shift an $\alpha$ portion of the weight of the dummy to towards the base predictor $P$, thus performing  \textit{``weight shifting''}. 

The following result implies an exponentially diminishing refusal probability in terms of the partial cumulative loss $L_{P,t,t_0}$, for an arbitrary $t_0 < t$, which implies that Weight-Shifting SafePredict can quickly recover to make predictions if the base predictor performs well starting time $t_0$. 

\begin{lemma} The probability of refusing to predict at time $t+1$, $w_{D,t+1}$, by the using the Weight-Shifting SafePredict (Algorithm 4 with $\beta_t = 1$ and $\alpha_t = \alpha$,  $\forall~t$) satisfies the following inequality
\begin{align}
 w_{D,t+1}  \leq  \frac{1-\alpha}{\alpha} e^{ \eta  \left(L_{P,t,t_0} - \epsilon' (t-t_0) \right)} \label{adapt_eff}
\end{align}
for $\epsilon' = \epsilon + \alpha/\eta$.
\end{lemma}

\begin{proof}
First, write the update rule eq. (\ref{ws_simplified}) in terms of the probability of refusal by noting $w_{D,t} = 1- w_{P,t}$
\begin{align}
w_{D,t+1} & = ~ w_{D,t} \left(1-\alpha \right)\frac{e^{-\eta \epsilon}}{w_{P,t}e^{-\eta l_{P,t}} + w_{D,t}e^{-\eta \epsilon}} \nonumber \\ 
 & = ~ w_{D,t} \left(1-\alpha \right)e^{\eta \left(m_t - \epsilon\right)} \label{a5}  \\
& = ~ w_{D,t_0+1} \left(1-\alpha \right)^{t - t_0} e^{\eta \left(\sum_{\tau = t_0+1}^t m_\tau - \epsilon(t-t_0)\right)} \label{b5}  \\
& \leq ~ w_{D,t_0+1}  e^{\eta \left(\sum_{\tau = t_0+1}^t m_\tau - \epsilon'(t-t_0)\right)}. \label{mid_step}
\end{align}
Where (\ref{a5}) follows by replacing the denominator with the definition of mix-loss from Section 3.1, (\ref{b5}) follows by recursing this update rule $t-t_0$ time, and the final step  follows from the inequality $1-\alpha \leq e^{-\alpha}$ for $0\leq \alpha \leq 1 $ and  $\epsilon' = \epsilon + \alpha/\eta$.

By Lemma 4.1, the mix-loss suffered by Algorithm 4 is equal to the one suffered by the EWAF over the virtual ensemble described in the the lemma. Note that we can consider all the virtual experts that follow $P$ from time $t_0+1$ to $t$ as a single super-expert since they suffer the same loss sequence, namely $l_{P,t_0+1}, \ldots, l_{P,t}$, within this interval. By denoting this super expert as $Q$, we can compute its total weight at time $t_0+1$ as
\[w_{Q,t_0+1} = \sum_{i: b_{i,t_0+1}=1}w_{P_i,t_0+1} ~ = ~ w_{P,t_0+1},\]
where equality to $w_{P,t_0+1}$ again follows from the Lemma 4.1. 
Then we can bound the sum in the eq. (\ref{mid_step}) using Lemma 3.2  for  virtual expert $Q$, 
\begin{eqnarray}
\sum_{\tau = t_0+1}^t m_\tau & \leq & L_{Q,t,t_0} - \frac{\log (w_{Q,t_0+1})}{\eta} \nonumber \\ & = & L_{P,t,t_0} - \frac{\log (w_{P,t_0+1})}{\eta}. \label{part_mix}
\end{eqnarray}
Finally, we conclude the proof by employing eq. (\ref{part_mix}) in eq. (\ref{mid_step}) and noting $w_{P,t_0+1} \geq \alpha$ and $w_{D,t_0+1} \leq 1-\alpha$, i.e.
\begin{eqnarray}
w_{D,t+1} 
& \leq & w_{D,t_0+1}  e^{\eta \left(L_{P,t,t_0} - \epsilon'(t-t_0)\right) - \log (w_{P,t_0+1})} \nonumber \\ 
& = & \frac{w_{D,t_0+1}}{w_{P,t_0+1}}  e^{\eta \left(L_{P,t,t_0} - \epsilon'(t-t_0)\right)} \nonumber \\
& \leq & \frac{1-\alpha}{\alpha}  e^{\eta \left(L_{P,t,t_0} - \epsilon'(t-t_0)\right)}. \nonumber
\end{eqnarray}
\end{proof}

We note that the dominant term of the RHS of eq. (\ref{adapt_eff}) is the exponential term since the preceding term $(1-\alpha)/\alpha$ for $\alpha = \Theta\left(1/T\right)$ increases only linearly with $T$. Therefore,  Lemma 4.2 implies that if $P$ starts to do well at time $t_0$, i.e. if its error rate starting from $t_0$ becomes less than the target rate $\epsilon$, the prediction probability will increase to $1$ exponentially fast.

\begin{algorithm}[!htb]
\caption*{ \textbf{Weight-Shifting SafePredict 
with Doub. Trick}\\
\textit{Base predictor}: $P$; ~\textit{Initial weight}: $w_{P,1} \in (0,1)$ \\
\textit{Target error rate}: $\epsilon \in (0,1)$; \textit{Adaptivitiy Parameter}: $\alpha \in [0,1)$}
\begin{algorithmic}[1]
\State Initialize $t = 1$
\For {each $k=1,2,\ldots$}
\State Reset $w_{P,t} = w_{P,1}$, ~ $V_{sum} = 0$, ~ and 
\[\eta = \sqrt{-\log\left(w_{D,1}\left(1-\alpha\right)^{T-1}\right)/\left(1-\epsilon \right)^2 / 2^{k}}\]
\While {$V_{sum} \leq 2^k$}
\State Predict with probability $w_{P,t}$, refuse otherwise,
\[\hat{y}_{t} =   \left\{
\begin{array}{ll}
      \hat{y}_{P,t} & ~~~\textrm{with prob. $w_{P,t}$}  \\
      \varnothing & ~~~\textrm{otherwise} \\
\end{array} 
\right. \]
\State Update the prediction probability:
\begin{eqnarray}
w_{P,t+1} & = &  \alpha + \left(1-\alpha\right) \frac{
w_{P,t}e^{-\eta l_{P,t}}}{
w_{P,t}e^{-\eta l_{P,t}} + w_{D,t}
e^{-\eta\epsilon}} \nonumber
\end{eqnarray}
\State Compute $V_{sum} \gets V_{sum} + w_{P,t+1}w_{D,t+1}$
\State Increment $t$ by $1$, i.e. $t \gets t+1$
\EndWhile
\EndFor
\end{algorithmic}
\end{algorithm}

As  in Section 3, the proposed learning rates for Adaptive SafePredict and therefore the Weight-Shifting SafePredict depend on $V^*$, and can be estimated using the doubling trick as described in Section 3.3. For the sake of completeness, pseudo-code for the Weight-Shifting SafePredict with the doubling trick is given in Algorithm 5. We can also extend the validity bound given in Corollary 4.1.2 by following the same steps we used in the proof of Theorem 3.6. 

\begin{corollary} For any $P$, $\epsilon < 1/2$, $0< w_{P,1} < 1$, and $0 \leq \alpha < 1$ Weight-Shifting SafePredict with the doubling trick given in Algorithm 5 satisfies 
\begin{eqnarray}
\frac{L_{P,T}^* }{T^*} & \leq & \epsilon + \left(1 - \epsilon\right)\frac{2\sqrt{2 V^*}}{\sqrt{2}-1}\frac{\sqrt{\log \left(1/w_{D,1}\right) + T\alpha + T\alpha^2}}{T^*}. \nonumber \end{eqnarray}

\end{corollary}

Note that for $\alpha = 0$, Algorithm 5 reduces to the original SafePredict (Alg. 3). As we increase  $\alpha$, i.e. the adaptivity,  the efficiency increases, since we increase the likelihood to make a prediction at each step. Furthermore, validity is guaranteed as long as $\alpha$ is on the order of $1/T$. In the next section, the impact of $\alpha$ is numerically evaluated.

\section{Experiments}

In this section, we investigate the performance of the proposed meta-algorithms on both synthetic and real  data.\footnote{For the sake of reproducibility, Python scripts used to generate our results, tabulated results on the synthetic data, and the experiments on other data sets are provided in the supplementary material.}
 In Section 5.1, we randomly generate loss sequences and verify the validity of our algorithms empirically for various loss statistics and various degrees of adaptivity. The experiments show that the Weight-Shifting SafePredict boosts the number of predictions in changing environments while preserving the validity of the algorithm.

In Section 5.2, we compare SafePredict with popular confidence-based refusal methods on the well-known MNIST digit recognition dataset \cite{lecun-mnisthandwrittendigit-2010}. SafePredict increases the efficiency relative to other refusal mechanisms while guaranteeing validity.

In the following, for the sake of brevity, we refer Weight-Shifting SafePredict (Alg. 5) as simply SafePredict by noting that $\alpha=0$ corresponds to the original SafePredict (Alg. 3).

\subsection{Synthetic Data}
In this subsection, we generate a binary sequence of loss-values with varying error probabilities. The subsection examines the validity and efficiency of SafePredict. In particular, we restrict the adaptivity parameter of the SafePredict to have the form $\alpha = k/T$, in accordance with Corollary 4.1.2, and observe the trade-off in choosing the parameter $k$. 

\begin{figure}[!h] 
\centering
\includegraphics[width=\linewidth]{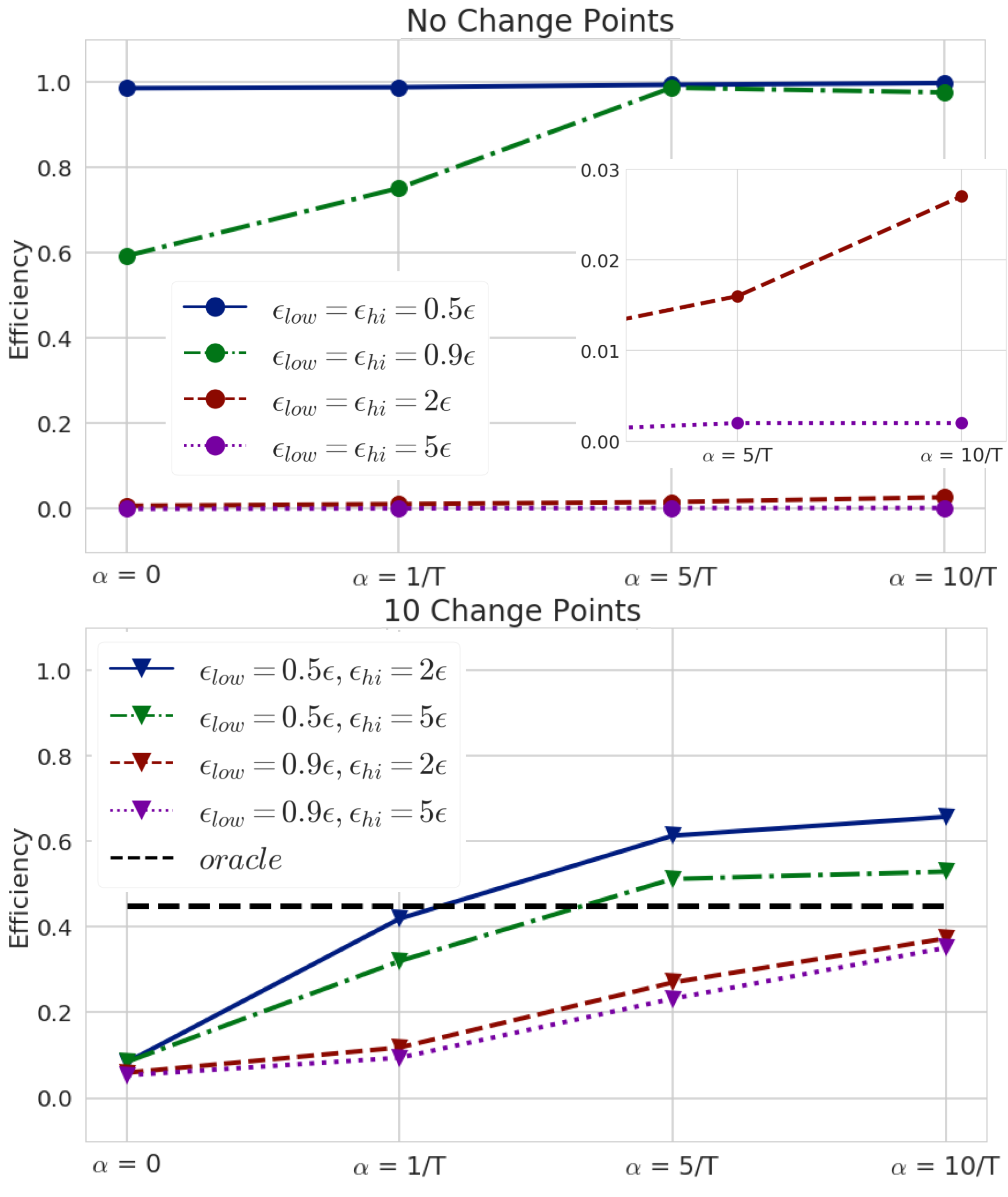}
\caption{\textit{Efficiency Experiments on Synthetic Data: The efficiency ($T^*/T$) of SafePredict with respect to increasing choices of $\alpha$. (top) If the base predictor has a constant error rate which is higher than the target, SafePredict almost always refuses. The number of predictions in this case increases with $\alpha$. (bottom) On the other hand, when the error rate of the base predictor fluctuates around the target, the efficiency of SafePredict increases as $\alpha$ increases and achieves nearly the same efficiency as the oracle, which predicts if only if $\epsilon_t \leq \epsilon$. No matter what, asymptotic validity is preserved.
}}\label{fig_2}
\end{figure}

\textbf{Parameters:} We fix the time horizon $T = 50000$, initial weight $w_{P,1} = 0.5$ and the target error rate $\epsilon = 0.05$. Then we evaluate our results  for $\alpha \in \{ 0, 1/T, 5/T, 10/T\}$. 

\textbf{Data Generation:} 
To evaluate the performance of the meta-algorithm, we assume the existence of a base predictor $P$ with a time varying error rate, and generate the loss sequence corresponding to its predictions randomly. To model the changes in the error-rate we employ a simple change-point model. 
The statistical properties of the generated loss sequence $l_{P,1}, \ldots, l_{P,T}$ are characterized by the following three parameters: low error level ($\epsilon_{low}$), high error level ($\epsilon_{hi}$), and the number of change points ($numChange$). 
To generate a particular loss sequence, we first split the time horizon into $numChange+1$ non-overlapping, consecutive, equal length blocks. Then we assume the error rate of $P$ to be constant within each block and alternates between $\epsilon_{low}$ and $\epsilon_{hi}$ for consecutive blocks. Formally, we generate each $l_{P,t}$ as an independent Bernoulli random variable as follows:
\[l_{P,t} =   \left\{
\begin{array}{ll}
      1 & ~~~\textrm{with prob. }\epsilon_{t} \\
      0 & ~~~\textrm{with prob. }1-\epsilon_{t} \\
\end{array} 
\right. ,\]
where 
\[\epsilon_{t} =   \left\{
\begin{array}{ll}
      \epsilon_{low} & ~~~\textrm{if } \left\lceil {t(numChange+1)}/{T} \right\rceil \textrm{ is even} \\
      \epsilon_{hi} & ~~~\textrm{otherwise} \\
\end{array} 
\right. .\]

\textbf{Results:} We generate 12 distinct loss sequences with different $numChange$, $\epsilon_{low}$ and $\epsilon_{hi}$ values and evaluate the error rate and efficiency of Algorithm 5 for various values of $\alpha$. The complete numerical results are presented in the supplementary material, but the key observations about the efficiency of SafePredict are summarized in Figures $2$ and $3$.

As a baseline for comparison, the results of an idealized oracle are also included. We assume the oracle has access to the true error rate of $P$, i.e. $\epsilon_t$, at each time point and decides to predict if and only if $\epsilon_t \leq \epsilon$. In other words, its prediction probability is equal to
\[w_{P,t} =   \left\{
\begin{array}{ll}
      1 & ~~~\textrm{if } \epsilon_t \leq \epsilon \\
      0 & ~~~\textrm{otherwise} \\
\end{array} 
\right. .\]

By contrast, in all our experiments, SafePredict does not know the number or location in time  of the change points. So the oracle enjoys a significant advantage. 

\begin{figure}[!htb] 
\centering
\includegraphics[width=\linewidth]{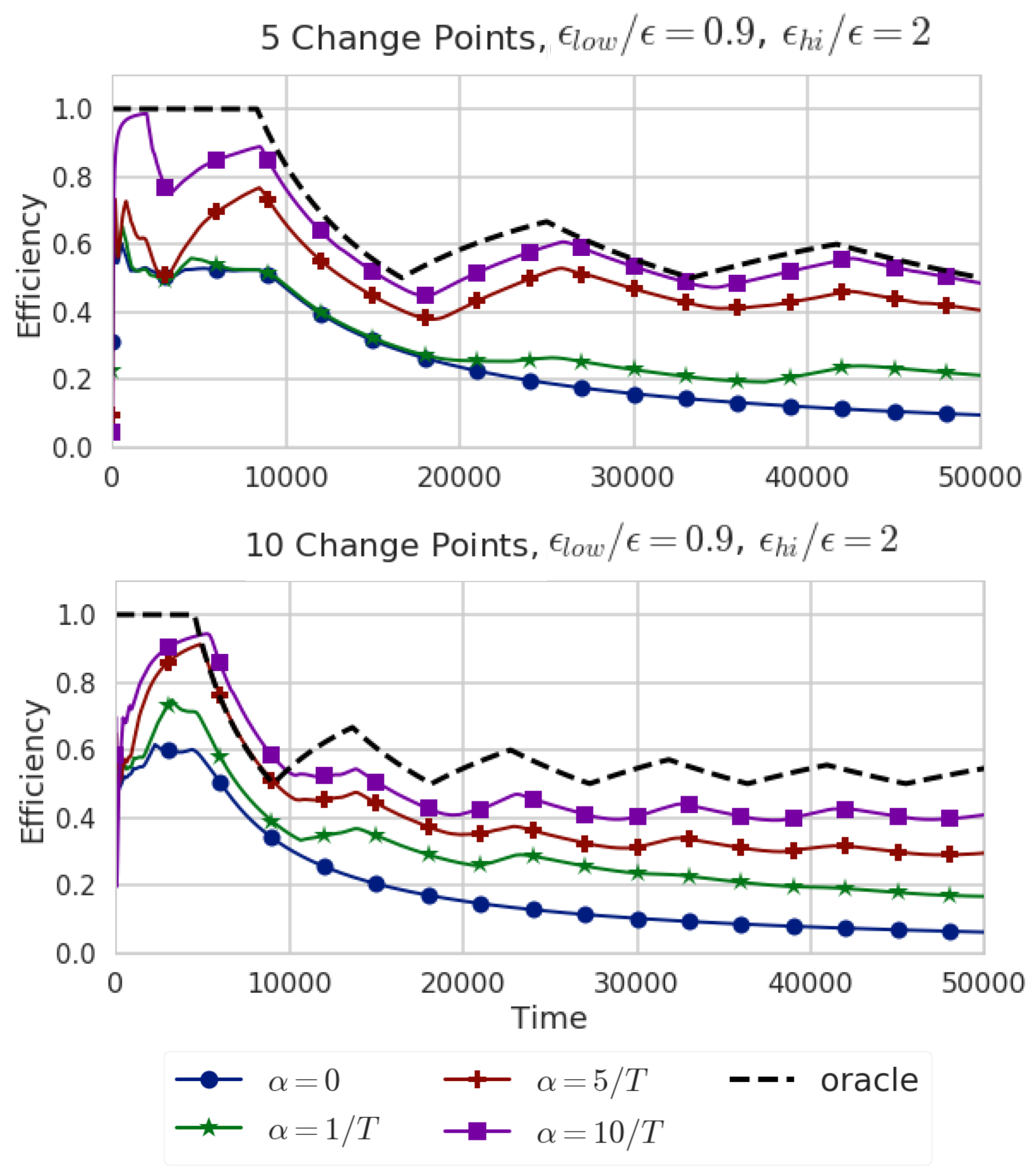}
\caption{\textit{Synthetic Data, Evolution of Efficiency:
Note $\alpha = 0$ corresponds to the original SafePredict (Alg. 3) and has no adaptivity.  For $\alpha>0$, SafePredict can track the change points and boost efficiency. Larger $\alpha$ implies better tracking. 
As the number of change points increases, SafePredict does a poorer job tracking the performance of the base predictor (relative to the oracle that knows the error rate), thus the efficiency drops. All the predictors in the figures are valid.}  
}
\label{fig_3}
\end{figure}

These simulations reveal two main issues:

\textit{1) Bound on Validity:} Following Corollary 4.1.2, for all $\alpha = O\left(1/T\right)$, the excess error rate is $O\left( {\sqrt{V^*}}/{T^*}\right)$. However, the constants hidden by the big-oh notation increase with $\alpha$, and become significant when $T^*$ is small. This effect can be observed most prominently in  experiments where the error rate of the base algorithm is consistently higher than the target rate. In these cases, the oracle always refuses as it should whereas SafePredict refuses often but not always. Asymptotically, SafePredict is still valid, but for finite sequences its error rate may exceed the target.  
On the other hand, in the experiments where the base predictor achieves the target for significant periods of time, the excess error rate of SafePredict stays within $7\%$ of the target error rate (below $0.0035$ for $\epsilon = 0.05$) and the efficiency increases with $\alpha$, see Table 1 in Supplementary Material.

\textit{2) Efficiency via adaptivity:} As expected from the theoretical analysis in Section 4 and empirically observed in Figure 2, the efficiency of SafePredict increases with $\alpha$. SafePredict performs nearly as well as the oracle, even though the oracle knows the true error probabilities and SafePredict does not. However, as the number of change points increases, SafePredict must refuse more to be able to adapt to the changes, and therefore suffers a drop in  efficiency. As can be seen in Figure 3, the efficiency of SafePredict decreases with the number of change points while the tracking ability of SafePredict increases with $\alpha$, sometimes approaching the efficiency achieved by the oracle.

\subsection{Real Data: MNIST Dataset}
We now explore the validity and efficiency of SafePredict on the MNIST digit recognition dataset \cite{lecun-mnisthandwrittendigit-2010} with a random forest classifier as the base predictor. Results  on other datasets from UCI repository \cite{Lichman:2013} are presented in Supplementary material. We also compare SafePredict with a natural confidence-based refusal mechanism.  
This method is widely used in practice, e.g.  \cite{zhang2013vehicle,hanczar2008classification,de2000reject}, and similar methods are used as baselines in the literature, see e.g. \cite{cortes2017line,bartlett2008classification}. Furthermore, one can conveniently make a fair comparison with SafePredict since both are meta-algorithms that can be used on top of (almost) any predictor.
Finally, we investigate a smart way of combining SafePredict with the confidence-based mechanism to improve the efficiency.

The MNIST dataset consists of $70000$ samples of handwritten digits where each sample is represented by a $784$ dimensional integer vector and labeled with a digit from $0$ to $9$. We randomly permute the data and choose the first $10000$ data points to use in our experiments. 
An artificial change-point is introduced at $t=5000$. We then choose a random label permutation and apply this permutation function to the last $5000$ data points, i.e. the second half of the data points are effectively chosen from another distribution. We fix the target error rate as $\epsilon = 0.08$ in our experiments.
 
Random forests are chosen due to their outstanding performance on a very broad set of tasks \cite{fernandez2014we} and robustness to the choice of parameters. We used a Python implementation of Random Forests from the Scikit-learn package \cite{pedregosa2011scikit} with default parameters. In the scope of this experiment, we retrain the random forest once every $100$ new data points in order to mitigate the computational burden.

The confidence-based refusal method we consider in this paper starts with a base predictor that outputs a confidence score for each prediction, and a refusal threshold. The meta-algorithm decides to refuse if the confidence score does not exceed the threshold value. In our experiments, we computed confidence scores via \textit{predict_proba} method for each prediction. Implementation details are available in the documentation of scikit-learn.
As in the case of retraining the base predictors, we update the refusal threshold once every $100$ data points. In particular, to update the threshold at time $t$ we use cross-validation over the points up to and including $t-1$ and choose the smallest threshold (i.e. the one refuses the least) that gives an error rate  smaller than $\epsilon$ over the non-refused predictions in the validation sets.

We show the results with SafePredict meta-algorithm (Alg. 5) with fixed parameters $\alpha = 1/1000$ ($10/T$) and $w_{P,1} = 0.5$. SafePredict is used on top of random forests either by itself or in conjunction with the confidence-based refusal meta-algorithm. In the latter multi-meta-algorithm scenario, SafePredict use the losses suffered by the confidence-based algorithm as an input. When the confidence-based meta-algorithm refuses to make a prediction for data point at time $t$, SafePredict also refuses to predict and does not update the weights, i.e. ignores the data point at time $t$. 

Finally, we include an ``amnesic adaptive" version of the combined meta-algorithm and base predictor that considers excessive refusals of SafePredict as a sign that adaptation is needed. To do so, the amnesic adaptive version monitors the fraction of the predictions made by the base predictor that are refused by SafePredict within the last epoch (i.e. $100$ data points). If the fraction is larger than $0.5$ the amnesic adaptive version ignores all data points from before the current epoch while both (i) training the base predictor (e.g. random forest) and (ii) any underlying confidence-based meta-algorithm. The evaluation of efficiency ($T^*/T$) and error rate ($L_{P,T}^*/T^*$) versus time is plotted for each of these predictors in Fig. 4.

\begin{figure}[!htbp]
\centering
\includegraphics[width=\linewidth]{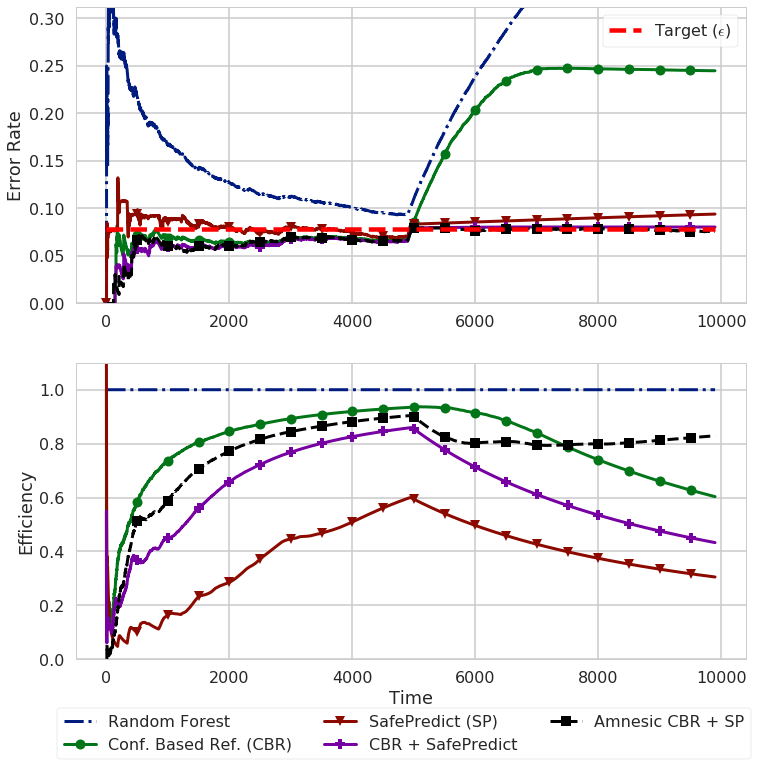}
\caption{\textit{MNIST Dataset: Efficiency is 1.0 for the base predictor but lower for the various refusing meta-algorithms. Validity is measured as a fraction of the target error rate. So the base predictor has a poor error rate (way over $\epsilon$). All the SafePredict variants rapidly approach a normalized error rate value of 1 though the error rate increases at the change point at time t = 5000. The confidence based competition cannot guarantee  asymptotic validity. Two forms of adaptivity help reduce the number of refusals: weight-shifting especially with a high $\alpha$ value and amnesic adaptivity. Combining both leads to the highest efficiency while preserving validity. }
}
\end{figure}

Experimental results lead to the following observations:

\textit{1) Validity:} As seen from the top subplot of Figure 4, the confidence-based refusal mechanism fails to satisfy the validity requirement of keeping the error rate below $\epsilon$ after the change point, i.e., $t>5000$. The reason is that confidence-base refusal requires data points to be (at least approximately) exchangeable to deliver the required error guarantee and this assumption fails after the change point. 
On the other hand, SafePredict establishes validity because it makes no assumptions about data points. 

\textit{2) Efficiency: } 
The robust validity of SafePredict relative to changes in the environment comes at a cost to efficiency. Generally, the confidence-based refusal method has a higher efficiency than SafePredict due to the nature of the refusal strategies (which follow from the stronger assumptions of the confidence-based methods). 
However, the discrepancy in the efficiency can be mitigated by employing the confidence-based mechanism as the base algorithm for SafePredict. This method implies a two-layered refusal mechanism, but as we discussed in Section 3.2 and
4.2, the second layer (SafePredict) will refuse seldom as long as the first layer (confidence-based algorithm) stays valid. As can be seen from Figure 4, this method (CBR+SafePredict) combines the best of confidence-based refusals and SafePredict by performing almost as efficiently as confidence-based refusals before the change point and preserving  validity throughout.

\textit{3) Amnesic Adaptivity: }  SafePredict on top of the confidence-based predictor remains valid throughout, but refusals increase after the change point, since the confidence-based mechanism is not valid anymore and thus causes excessive errors. In the amnesic adaptive variant, we use the excessive refusals to trigger an update of the base algorithm. Specifically, if the number of data points predicted by confidence-based predictor but refused by SafePredict is large, the amnesic approach concludes that the confidence-based algorithm is no longer well-calibrated, so earlier data points should be ignored. We denote this adaptive method as ``Amnesic CBR+SP''. As seen in the plots, Amnesic CBR+SP gives the most favorable performance in our experiments by preserving validity thanks to SafePredict and achieving better efficiency after the change point by forcing the confidence-based algorithm to forget the previous data points.

\section{Conclusion}

We have introduced a meta-algorithm, SafePredict, that works with any base prediction algorithm and asymptotically guarantees an upper bound on the error rate for non-refused predictions.  The error guarantee achieved by SafePredict does not depend on any assumption on the data or the base prediction algorithm. To achieve this, we refined the regret notion from the expert advice framework and recast the exponentially weighted average forecasting algorithm to be used as a method to manage refusals. 

To avoid too many refusals in changing environments, we  introduced a weight-shifting heuristic that encourages predictions when  the quality of the base predictor improves. We have also used an amnesic adaptation mechanism to further improve versatility in the face of occasional change points.  Our experiments show that these methods establish validity even in challenging environments while refusing seldom when the base predictor does well.

\ifCLASSOPTIONcompsoc
  \section*{Acknowledgments}
\else
  \section*{Acknowledgment}
\fi
Work supported in part by NYU Seed Grant, NYU WIRELESS, and the National Science Foundation grants CNS-1302336, MCB-1158273, IOS-1339362, MCB-1412232.

\ifCLASSOPTIONcaptionsoff
  \newpage
\fi

\bibliographystyle{IEEEtran}

 \begin{IEEEbiography}[{\includegraphics[width=25mm,height=32mm,clip,keepaspectratio]{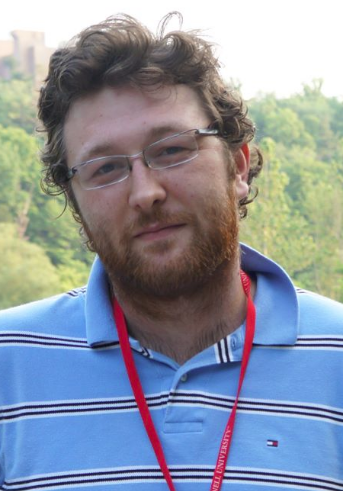}}]{Mustafa A. Kocak}
is a Ph.D. candidate at NYU Tandon School of Engineering, with ECE department. His research interests include statistical learning theory, information theory and wireless communications. Kocak
received a B.Sc. in electrical engineering  from Bilkent University, Ankara,Turkey. He is a student member of IEEE. Contact him at kocak@nyu.edu.
 \end{IEEEbiography}

\begin{IEEEbiography}[{\includegraphics[width=25mm,height=32mm,clip,keepaspectratio]{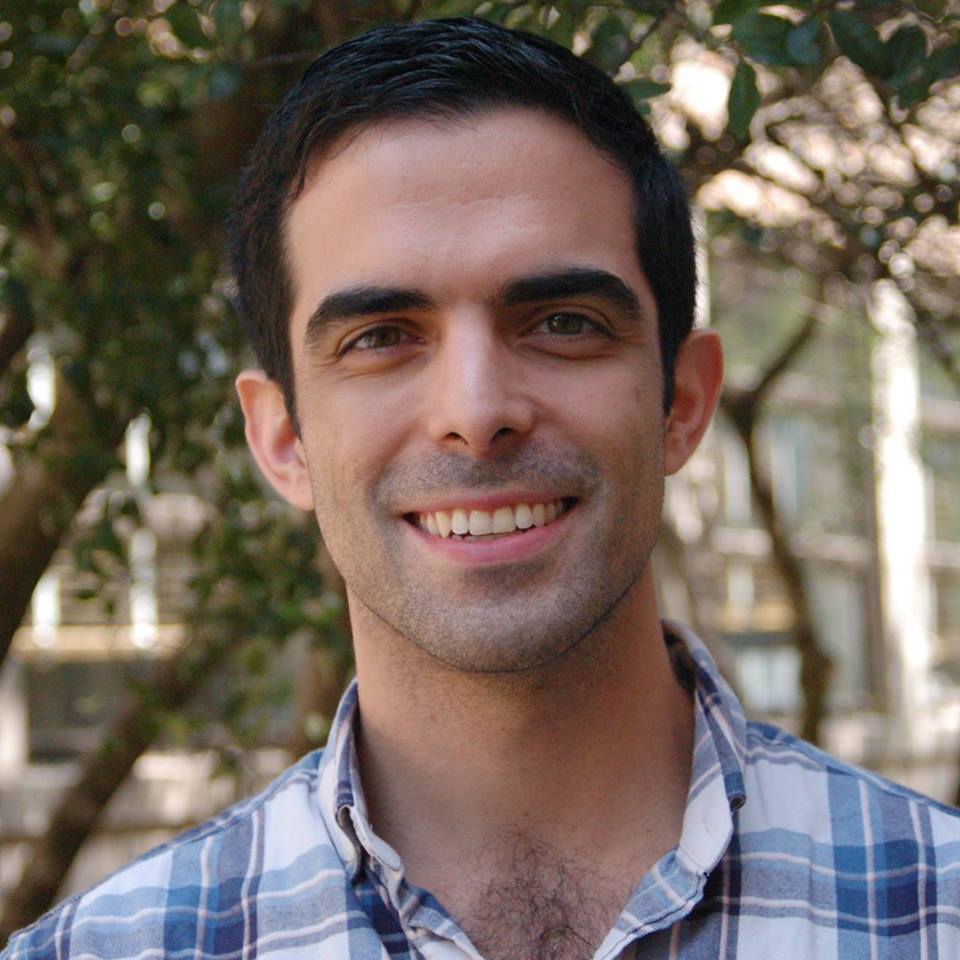}}]{David Ramirez} 
received a B.S. with honors in Engineering Physics from Tecnologico de Monterrey (ITESM), and M.S. and Ph.D. degrees in Electrical and Computer Engineering from Rice University. He is currently a Postdoctoral Researcher at New York University and a Visiting Postdoctoral Researcher at Princeton University. His research interests are in wireless networks, communication theory, \& optimization. 
\end{IEEEbiography}

\begin{IEEEbiography}[{\includegraphics[width=25mm,height=32mm,clip,keepaspectratio]{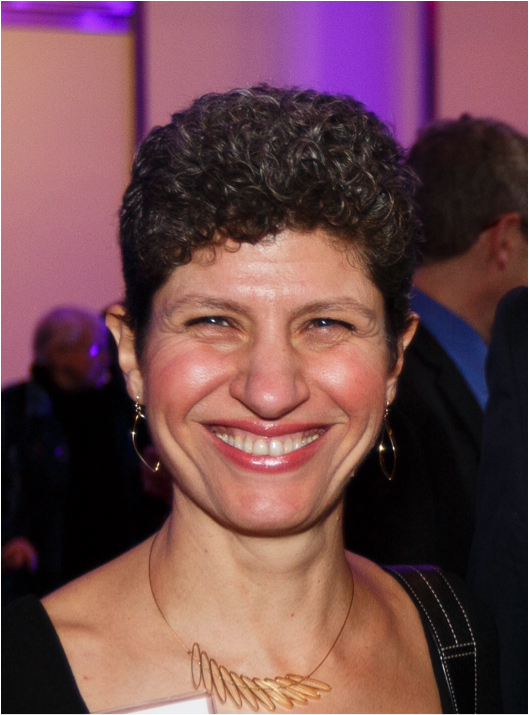}}]{Elza Erkip} received the B.S. degree in Electrical and Electronics Engineering from Middle East Technical University, Ankara, Turkey, and the M.S. and Ph.D. degrees in Electrical Engineering from Stanford University, Stanford, CA, USA. 
Currently, she is a Professor of Electrical and Computer Engineering with New York University Tandon School of Engineering, Brooklyn, NY, USA. 
Her research interests are in information theory, communication theory, and wireless communications.

Dr. Erkip is a member of the Science Academy Society of Turkey and is among the 2014 and 2015 Thomson Reuters Highly Cited Researchers. She received the NSF CAREER award in 2001 and the IEEE Communications Society WICE Outstanding Achievement Award in 2016. Her paper awards include the IEEE Communications Society Stephen O. Rice Paper Prize in 2004, and the IEEE Communications Society Award for Advances in Communication in 2013. She has been a member of the Board of Governors of the IEEE Information Theory Society since 2012 where she is currently the First Vice President. She was a Distinguished Lecturer of the IEEE Information Theory Society from 2013 to 2014. 
\end{IEEEbiography}

\begin{IEEEbiography}[{\includegraphics[width=25mm,height=32mm,clip,keepaspectratio]{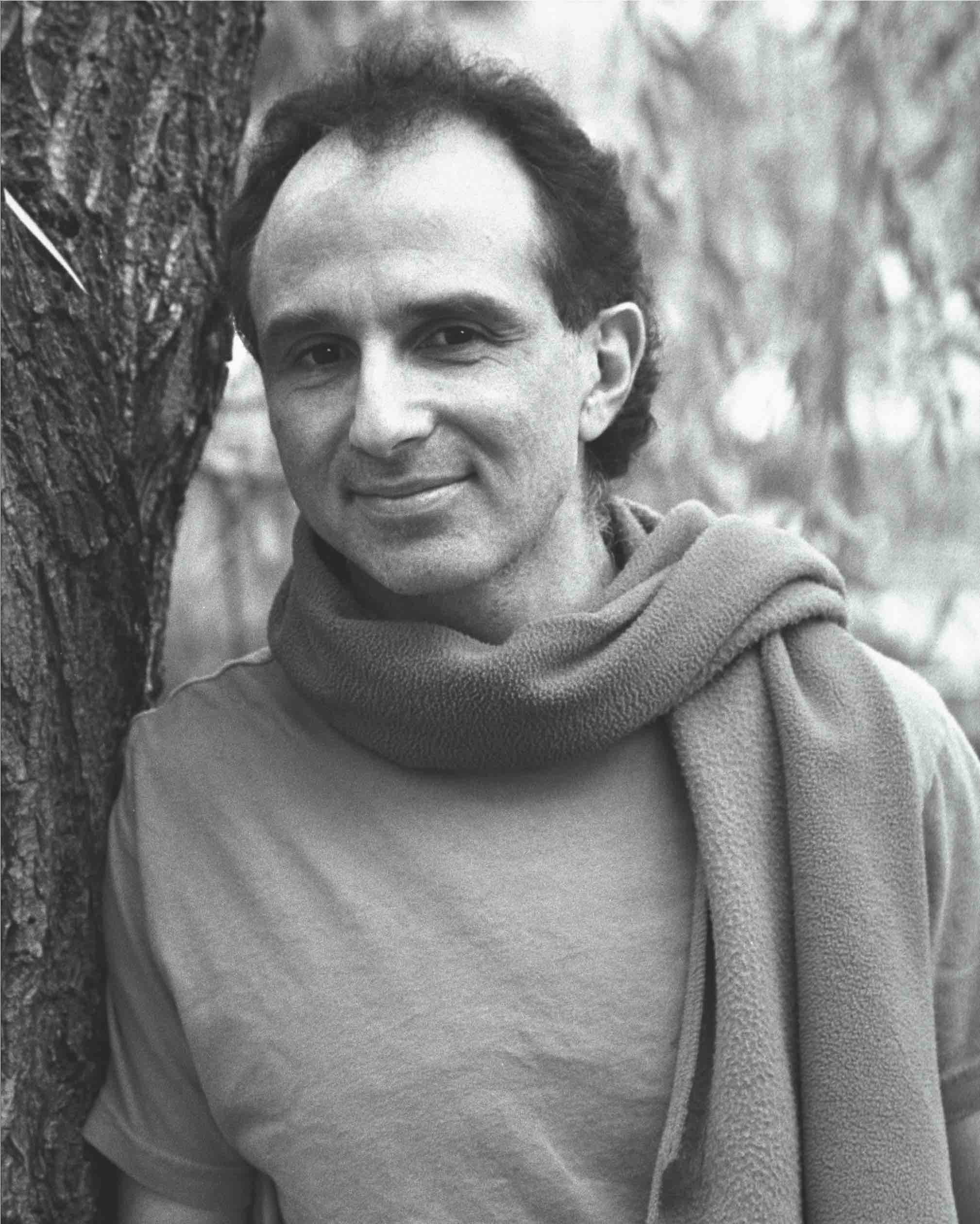}}]{Dennis E. Shasha}
 is a Professor of Computer Science at the Courant Institute of New York University. His research interests include data science, biological computing, wireless communication and concurrent data structures. Shasha 
received a PhD in applied math from Harvard University. He is an ACM Fellow and the recipient of an INRIA Internaional Chair. He is co-editor in chief of Information Systems, and is or has been the puzzle columnist for CACM and Scientific American. Contact shasha@cims.nyu.edu
\end{IEEEbiography}

\end{document}